\documentclass{article}

\RequirePackage{amsthm,amsmath,amsfonts,amssymb}
\RequirePackage[numbers,sort&compress]{natbib}
\RequirePackage[colorlinks,citecolor=blue,urlcolor=blue,linkcolor=blue]{hyperref}
\RequirePackage{graphicx}

\usepackage{float}
\usepackage{bbm}
\usepackage{comment}
\usepackage{mathrsfs}
\usepackage{dsfont}
\usepackage{latexsym}
\usepackage{exscale}
\usepackage[latin1]{inputenc}
\usepackage{mathtools}
\usepackage{color}
\usepackage{multirow}
\usepackage[ruled]{algorithm2e}
\usepackage{tcolorbox}
\usepackage{csquotes}
\usepackage{subcaption}
\usepackage{mwe}
\usepackage{tikz}
\usepackage{pstricks}
\usepackage{booktabs}
\usepackage{appendix}
\usepackage{bm}
\usepackage{url}
 \usepackage[margin=1in,top=1in,bottom=1in]{geometry}

\numberwithin{equation}{section}
\theoremstyle{plain}
\newtheorem{theorem}{Theorem}[section]
\newtheorem{lemma}[theorem]{Lemma}

\newtheorem{corollary}[theorem]{Corollary}
\theoremstyle{remark}
\newtheorem{definition}[theorem]{Definition}

\DeclareMathOperator*{\argmin}{arg\,min}
\DeclareMathOperator*{\diag}{diag}
\DeclareMathOperator*{\sign}{sign}

\title{Variable Selection and Regularization via Arbitrary Rectangle-range Generalized Elastic Net}
\author{Yujia Ding\footnote{Corresponding author.}\\ Institute of Mathematical Sciences\\
      Claremont Graduate University,
    California, USA\\
Email: yujia.ding@cgu.edu\\ \\
Qidi Peng\\ Institute of Mathematical Sciences\\
Claremont Graduate University, California, USA\\
Email: qidi.peng@cgu.edu\\\\
Zhengming Song\\ Institute of Mathematical Sciences\\
Claremont Graduate University, California, USA\\
Email: zhengming.song@cgu.edu\\ \\
Hansen Chen\\ Old Mission Capital\\
      Chicago, Illinois, USA.\\
Email: hansenchen1@gmail.com}
\date{}
\begin{document}
\maketitle

\begin{abstract}
We introduce the arbitrary rectangle-range generalized elastic net penalty method, abbreviated to ARGEN, for performing constrained variable selection and regularization in high-dimensional sparse linear models. As a natural extension of the nonnegative elastic net penalty method, ARGEN is proved to have variable selection consistency and  estimation consistency under some conditions.  The asymptotic behavior in distribution of the ARGEN estimators have been studied. We also propose an algorithm called MU-QP-RR-W-$l_1$ to efficiently solve ARGEN. By conducting simulation study we show that ARGEN outperforms the elastic net in a number of settings. Finally an application of S\&P 500 index tracking with constraints on the stock allocations is performed to provide general guidance for adapting ARGEN to solve real-world problems.
\end{abstract}

\textbf{MSC2020 subject classifications:}
Primary 62J07, 62F12; secondary 62P05


% Main text entry area

\section{Introduction}

Variable selection and regularization are essential tools in high-dimensional data analysis. Many existing strategies are able to achieve both high prediction accuracy and interpretability. For instance,  the lasso \citep{tibshirani} was popularized thanks to its computational efficiency \citep{efron1}, variable selection consistency \citep{zhao}, and estimation consistency \citep{negahhan}. We refer to \cite{zou,bickel,efron2,lounici,yuan,rocha,wang} for more in-depth discussions of lasso. Later, elastic net \cite{zouhas} is proposed through linearly combining the lasso and ridge regression-like penalties. As a more flexible model, elastic net is shown to be able to  outperform the lasso for high-dimensional data. 

Consider the following linear model
\begin{equation}
    \label{example_model}
    Y=\beta_0+\beta_1X_1+\ldots+\beta_pX_p+g(\beta_1,\ldots,\beta_p)+\epsilon,
\end{equation}
where $Y\in\mathbb R$ is the response variable, $X_1,\ldots,X_p\in\mathbb R$ are $p$ predictors, $g$ is some penalty function and $\epsilon$ is the residual. In the setting of ordinary linear models,  one considers no range constraint on the coefficients, i.e., it is assumed that  $\beta_1,\ldots,\beta_p\in\mathbb R$. However in practice $\beta_1,\ldots,\beta_p$ are often restricted to a prior range of values. For example, in portfolio management problem, the coefficients are considered as allocations of assets in a fund, which are valued in $[0,1]$; in academic grading problem, the coefficients are interpreted as weights of a list of courses, which are also ranged in $[0,1]$. Such constraints may influence the behavior of the penalty $g$,  as well as the estimated values of $\beta_1,\ldots,\beta_p$. Concerning adapting to this real world constraint, Wu et al. \cite{wu_nl} and Wu and Yang \cite{wu_nen} introduced nonnegative lasso and nonnegative elastic net approaches, which are successfully applied to solve the real world index tracking problem without short sales (this corresponds to the nonnegative-value constraint on weights).  There exist more such range constraints on the regression coefficients in the real world problems, therefore  more flexible models are needed to address problems that require arbitrary-range constraints on the regression coefficients. With this motivation, our first goal is to suggest a novel method that concerns arbitrary rectangle-range constraint on the regression coefficients. Secondly, as there exist necessary conditions for the variable selection to be consistent,  certain scenarios exist where the lasso or elastic net is inconsistent for variable selection. Recall that Zou \cite{zou}, Mouret et al. \cite{mouret2013generalized} and Sokolov et al. \cite{sokolov}  introduced methods that generalize the lasso and elastic net respectively by placing adaptive weights on the predictors and penalties. Therefore we will adopt this setting in our model, i.e., the coefficients in the penalties will be weighted. As conclusion, our paper proposes a method that allows arbitrary rectangle-range constraints on the regression coefficients of a generalized elastic net and provide rigorous theoretical results to support the consistency of the method. To elaborate, the proposed  \textit{arbitrary rectangle-range generalized elastic net method} (abbreviated to ARGEN), is a regularization method that deals with high-dimensional problems, and generalizes the nonnegative elastic net. Compared with nonnegative elastic net, ARGEN allows adding arbitrarily lower and upper constraints on the coefficients, and  considers effects of individual and interactive penalty weights. The former setting ensures ARGEN more adaptability to real world constraints, and the latter setting promises better performance due to its larger parameter searching space. Variable selection consistency and estimation consistency are then derived. To solve ARGEN, we introduce a novel algorithm \textit{multiplicative updates for solving  quadratic programming with rectangle range and weighted $l_1$ regularizer} (abbreviated to MU-QP-RR-W-$l_1$).  We summarize the main contributions of our paper as follows:
\begin{enumerate}
    \item We introduce ARGEN, a method of solving variable selection and regularization problems that require the regression coefficients to be ranged in some rectangle in $\mathbb R^p$ (see (\ref{model})). As a flexible approach, ARGEN includes the nonnegative elastic net and a number of new extensions of the models lasso, ridge, and elastic net.
    \item Subject to some condition on the inputs, the variable selection consistency, the estimation consistency, and the limiting distribution of the estimator of ARGEN are obtained. We refer to Theorems \ref{thm:VSC}, \ref{thm:estimation_consistency}, and \ref{thm:convergence_law}.
    \item A novel algorithm MU-QP-RR-W-$l_1$ is introduced to solve the general quadratic programming problem 
$\min_{v\in[0,  l]} F(v)$ $= \frac{1}{2} v'Av + b' v+d'|v-v^0|$, following the notations in (\ref{4}). The algorithm is implemented as a Python library through the PyPi server\footnote{\url{https://pypi.org/project/generalized-elastic-net}} and is publicly shared. 
\item We show a successful real world application of the ARGEN approach in the S\&P 500 index tracking problem. Readers can get full access to the Python script in the Github repository\footnote{\url{https://github.com/songzhm/arbitraryElasticNet}}.
\end{enumerate} 

Throughout the paper, we denote the transpose of a matrix by $(\bullet)'$, the $i$-th column of a matrix by $(\bullet)_i$, the entry in the $i$-th row and $j$-th column of a matrix by $(\bullet)_{ij}$, the diagonal matrix with diagonal vector $\mathrm{\textbf{x}}$ by $\diag(\mathrm{\textbf{x}})$, and the maximum (resp. minimum) element of a vector by $\max(\bullet)$ (resp. $\min(\bullet)$). Besides, an $n\times n$ matrix $X$ can be expressed by $ X=(X_{ij})_{1\le i,j\le n}$. The elementwise absolute value of a vector or matrix is $|\bullet|$; for $\mathrm{\textbf x}=(x_1, \ldots, x_p)$,
$|\mathrm{\textbf x}|:=(|x_1|,\ldots,|x_p|)$ and for an $n\times n$ matrix $X$, $|X|:=(|X_{ij}|)_{1\le i,j\le n}$. Moreover, let $\mathrm{\textbf x} =(x_1,\ldots,x_p)$, $\mathrm{\textbf y}=(y_1,\ldots,y_p)$ be two equal-length vectors, we denote the $p$-dimensional interval by $[\mathrm{\textbf x}, \mathrm{\textbf y}]:=[x_1,y_1]\times\ldots\times[x_p,y_p]$.

In the sequel, we consider the linear regression model
\begin{equation}
\label{1}
Y = X\beta^* + \epsilon,
\end{equation}
where $X$ is a deterministic $ n \times p $ design matrix, $Y = (y_1~\ldots~y_n)'$ is an $n\times1$ response vector and $\varepsilon = ( \varepsilon_1~\ldots~\varepsilon_n)' $ is a  Gaussian noise with marginal variance $\sigma^2 $. Without loss of generality, we assume all the $p$ predictors are real-valued and centered, so the
intercept can be ignored.  $\beta^* \in \mathbb R^{p}$ is the regression coefficients.

The rest of the paper is organized as follows. In Section \ref{sec::argen}, we discuss of the analytical features of ARGEN and discuss of its variable selection consistency (Theorem \ref{thm:VSC}), estimation consistency (Theorem \ref{thm:estimation_consistency}) and estimator's limiting distribution (Theorem \ref{thm:convergence_law}). In Section \ref{sec:numerical_scheme}, we propose an efficient algorithm MU-QP-RR-W-$l_1$ for solving ARGEN. Approaches we use to speed up hyper-parameter optimization are discussed in Section \ref{Hyper_parameter_tuning}. Simulations that compare the performances of various methods are conducted in Section \ref{sec:simulation}. Section \ref{sec:real_world} shows an application of ARGEN to the real world S\&P 500 index tracking problem. Section \ref{sec:conclusion} is devoted to the conclusion and discussion of future research. Technical proofs are moved to Appendix.

\section{The ARGEN}
\label{sec::argen}
\subsection{Definition}
In practice it is often natural to assume sparsity in the high-dimensional dataset problem. Therefore in the sequel we assume that the linear model (\ref{1}) is $q$-sparse, i.e., $\beta^*$ has at most $ q ~(q \ll p) $ nonzero elements. We intend to cope with the case when there is a control on the range of the coefficients, that is, let $ s =(s_1,\ldots,s_p)$, $ t=(t_1,\ldots,t_p)$ with $s_i\in \mathbb R\cup\{-\infty\},~t_i\in \mathbb R\cup\{+\infty\}$, $s_i< t_i$ for all $i=1,\ldots,p$, the optimal coefficients are in a $p$-dimensional rectangle
$\mathcal I :=[s, t]\subset \mathbb R^p$. To capture the penalty weights
for individual features, we introduced  $\mathrm{w}_{n} = (\mathrm{w}_{n,1}~ \ldots~ \mathrm w_{n,p})'$ as weights for each coefficient in the $l_1$ penalty, and it satisfies $\mathrm w_{n,i}\geq0, i=1,\cdots,p$. In addition to individual features, $\Sigma_n$, a positive semi-definite matrix, is introduced to represent the penalty weights for interactions between any two features. Consider the linear model (\ref{1}) and let $ \beta=(\beta_1~\ldots~\beta_p)'$ be a vector in $\mathbb R^p$. The  ARGEN  estimator of $\beta$ is given by 
\begin{equation}
\label{model}
\widehat{\beta}(\lambda_n^{(1)},\lambda_n^{(2)},\mathrm w_n, \Sigma_n) =\argmin_{\beta \in\mathcal I}\left( \left\| Y - X\beta \right\|_2^2+\lambda_n^{(1)}  \mathrm w_n'|\beta|+\lambda_n^{(2)} \beta'\Sigma_n \beta\right).
\end{equation}
Here $\lambda_n^{(1)},\lambda_n^{(2)}\ge0$ are the  tuning parameters which control the importance of the $l_1$ and $l_2$ regularization terms, respectively.

The ARGEN (\ref{model}) naturally extends the elastic net method. That is, it becomes the elastic net when $\mathcal I=\mathbb R^p$, $\mathrm w_n=(1~ \ldots~ 1)'$, and $\Sigma_n$ is the identity matrix. Thus ARGEN extends  the
lasso and ridge methods by further assigning $\lambda_n^{(2)}=0$ and $\lambda_n^{(1)}=0$ respectively. In addition, ARGEN becomes the nonnegative elastic net if we replace $\mathcal I=\mathbb R^p$  with $\mathcal I=\mathbb R_+^p:=[0,+\infty)^p$ in the setting of elastic net.

\subsection{Variable Selection Consistency}
\label{sec:variable_selection}
We define the variable selection consistency for the ARGEN as follows. For $i=1,\ldots,p$, we decompose the interval $[s_i,t_i]$ with $s_i<t_i$ into seven disjoint sub-intervals:
\begin{equation*}
    % \label{I:split}
    [s_i,t_i]=\bigcup_{k=2}^6\mathcal G_i^{(k)}\bigcup\mathcal G_i^{(1-)}\bigcup\mathcal G_i^{(1+)},
\end{equation*}
where
\begin{eqnarray*}
&&\mathcal G_i^{(1-)}=(s_i, t_i)\cap(-\infty,0),~~~~~\mathcal G_i^{(1+)}=(s_i, t_i)\cap(0,+\infty),\nonumber\\
&&\mathcal G_i^{(2)}=\{s_i\}\backslash\{0\},~~~~~\mathcal G_i^{(3)}=\{t_i\}\backslash\{0\},\nonumber\\
&&\mathcal G_i^{(4)}=\{s_i\}\cap\{0\},~~~~~\mathcal G_i^{(5)}=\{t_i\}\cap\{0\},~~~~~\mathcal G_i^{(6)}=(s_i,t_i)\cap\{0\}.
\end{eqnarray*}
In addition, we define 
$
% \label{def:G_1}
\mathcal G_i^{(1)}=\mathcal G_i^{(1-)}\cup\mathcal G_i^{(1+)}
$ for simplicity.
Correspondingly, each coefficient in $\beta^*$ belongs to one of the seven groups of values; i.e., for each $i=1,\ldots,p$, there is a unique $k_i\in\{1-,1+,2,\ldots,6\}$ such that 
$
% \label{def:true_beta}
\beta_i^*\in \mathcal G_i^{(k_i)}.
$
Now for $j\in\{1-,1+,2,\ldots,6\}$, denote by 
\begin{equation*}
    % \label{def:S}
    S_{(j)}=\left\{i\in\{1,\ldots,p\}:~\beta_i^*\in \mathcal G_{i}^{(j)}\right\},
\end{equation*}
the set of indexes $i$ for which $\beta_i^*$ belongs to the $j$-th group of values, and let $\# S_{(j)}$ be the cardinality of the set. Correspondingly, we can define
\begin{equation*}
    % \label{def:S_hat}
    \widehat{S}_{(j)}(\lambda_n^{(1)},\lambda_n^{(2)},\mathrm w_n,\mathrm\Sigma_n)=\left\{i\in\{1,\ldots,p\}:~\widehat\beta_i\in \mathcal G_{i}^{(j)}\right\}.
\end{equation*}
\begin{definition}
  	\label{def:VSC}
  	 ARGEN (\ref{model}) is said to have variable selection consistency if there exist $\lambda_n^{(1)}$, $\lambda_n^{(2)}$, $\mathrm w_n$, and $\Sigma_n$ such that
  	\begin{equation}
  	\label{VSC}
  		\mathbb P\left(\widehat{S}_{(j)}\big(\lambda_n^{(1)},\lambda_n^{(2)},\mathrm w_n,\mathrm\Sigma_n\big) = S_{(j)}\Big|S_{(j)}\ne\emptyset\right) \xrightarrow[n\to\infty]{}1~\mbox{for $j\in\{1-,1+,2,\ldots,6\}$}.
  	\end{equation}
\end{definition}
(\ref{VSC}) implies that, starting from some $n$, it is of big opportunity that $\widehat\beta_i$ equals $\beta_i^*$ if $\beta^*_i\in\{0,s_i,t_i\}$. Such property includes the variable selection consistency of the nonnegative elastic net and  elastic net as particular cases. Therefore our definition of the variable selection consistency for ARGEN is in a broader sense than that for the \enquote{free-range} or nonnegative elastic net  \citep{zhao, wu_nl,wu_nen}.  

Let $X_{(1)}=(X_{(1-)}, X_{(1+)})$ and for $j\in\{1-,1+,2,\ldots,6\}$, let $X_{(j)} = \left(X_i\right)_{i\in S(j)}$  be the observed  predictor values corresponding to the $j$th group of indexes. Similarly, let $\beta^*_{(j)} = \left(\beta^*_i\right)_{i\in S(j)}$, $s_{(j)} = \left(s_i\right)_{i\in S(j)}$, $t_{(j)} = \left(t_i\right)_{i\in S(j)}$, $\mathrm w_{n, (j)} = \left(\mathrm w_{n, i}\right)_{i\in S(j)}$, and $\Sigma_{n, (j_1j_2)} = \left(\Sigma_{n, i_1,i_2}\right)_{i_1\in S(j_1), i_2\in S(j_2)}$. Moreover, let $C$ be
\begin{equation}
\label{def:C}
  	C :=\begin{pmatrix}
  	C_{ij} 
  	\end{pmatrix}_{1\le i,j\le 6}=\frac{1}{n}X'X=\begin{pmatrix}
  	\frac{1}{n}X'_{(i)} X_{(j)}
  	\end{pmatrix}_{1\le i,j\le6}
\end{equation}
and $\Lambda_{\min}(C_{11})$ be the minimal eigenvalue of $C_{11}$. Denote by
\begin{equation}
\label{rho_C}
\begin{aligned}
    &\rho_n^{(1)}:=\max\bigg\{\Big(C_{11}+\frac{\lambda_n^{(2)}}{n}\Sigma_{n,(11)}\Big)^{-1}C_{11}\beta_{(1)}^*- t_{(1)}\bigg\},\\
    &\rho_n^{(2)}:=\min\bigg\{\Big(C_{11}+\frac{\lambda_n^{(2)}}{n}\Sigma_{n,(11)}\Big)^{-1}C_{11}\beta_{(1)}^*- s_{(1)}\bigg\},\\
    &C_n := \Big(C_{11}+\frac{\lambda_n^{(2)}}{n}\Sigma_{n,(11)}\Big)^{-1}\Big(\frac{\lambda_n^{(1)}}{2}
\diag(\sign(\beta_{(1)}^*))\mathrm w_{n,(1)}+\lambda_n^{(2)}(\Sigma_{n,(12)} s_{(2)}+\Sigma_{n,(13)} t_{(3)})\Big),\\
    &C_{n}^{\max}:=\max C_n,~~~C_{n}^{\min}:=\min C_n,
\end{aligned}
\end{equation}
where for a vector $v=(v_1,\ldots,v_n)$, $\sign (v):=(\sign(v_1),\ldots,\sign(v_n))$ denotes the vector of signs of the elements in $v$, and $\diag(v)$ denotes the diagonal matrix with diagonal elements $v$.  
To show ARGEN admits the variable selection consistency (\ref{VSC}), we  assume that the following conditions hold: 
\begin{eqnarray}
    \label{con_0}
    &&q > 1,~~~~~p-q > 1,\\
    \label{con_1}
    &&\frac{\lambda_n^{(1)}}{\sqrt{n}}\xrightarrow[n\to\infty]{}+\infty,\\
    \label{con_4}
    &&\frac{1}{n}\max\limits_{1\le i\le p}X_i'X_i\xrightarrow[n\to\infty]{}0,\\
    \label{con_5}
    &&\max_{j\in\{2,\ldots,6\}}\Sigma_{n,(11)}^{-1}\Sigma_{n,(i1)}=\mathcal{O}(1),~\mbox{as}~n\to\infty,
\end{eqnarray}
and for $j\in\{1-,1+\}$,
\begin{eqnarray}
    \label{con_2}
   &&\frac{1}{\rho_n^{(1)}}\left(\frac{8\sigma\sqrt{\#S_{(1)}\operatorname{trace(C_{11})}\log (\#S_{(j)})}}{n\Lambda_{\min} (C_{11}+\lambda_n^{(2)}\Sigma_{n,(11)}/n)}+\frac{|C_n^{\min}|}{n}\right)\xrightarrow[n\to\infty]{}0,\\
   \label{con_3}
    &&\frac{1}{\rho_n^{(2)}}\left(\frac{8\sigma\sqrt{\#S_{(1)}\operatorname{trace(C_{11})}\log (\#S_{(j)})}}{n\Lambda_{\min} (C_{11}+\lambda_n^{(2)}\Sigma_{n,(11)}/n)}+\frac{|C_n^{\max}|}{n}\right)\xrightarrow[n\to\infty]{}0.
\end{eqnarray}
Besides, we assume that the \textit{arbitrary rectangle-range elastic irrepresentable condition}  (AREIC), defined below, is satisfied.
   \begin{definition}
   	\label{AREIC}
   	The AREIC is given as: For $j=2,\ldots,6$ satisfying $S_{(j)}\ne\emptyset$, there exists a positive constant vector $\eta_{(j)}$, such that
	\begin{equation}
\begin{aligned}
	\label{AREIC_ineq}
	&\Big(C_{j1}+\frac{\lambda_n^{(2)}}{n}\Sigma_{n,(j1)}\Big)\Big(C_{11}+\frac{\lambda_n^{(2)}}{n}\Sigma_{n,(11)}\Big)^{-1}\Big(\diag(\sign(\beta_{(1)}^*))\mathrm w_{n,(1)}\\
&\hspace{1.5cm}+\frac{2\lambda_n^{(2)}}{\lambda_n^{(1)}}\Sigma_{n,(11)}\beta_{(1)}^*+\frac{2\lambda_n^{(2)}}{\lambda_n^{(1)}}\left(\Sigma_{n,(12)} s_{(2)}+\Sigma_{n,(13)} t_{(3)}\right)\Big)\\
&\hspace{1.5cm}-\frac{2\lambda_n^{(2)}}{\lambda_n^{(1)}}\Sigma_{n,(j1)}\beta_{(1)}^*-\frac{2\lambda_n^{(2)}}{\lambda_n^{(1)}}\left(\Sigma_{n,(j2)} s_{(2)}+\Sigma_{n,(j3)} t_{(3)}\right)
\left\{
\begin{array}{ll}
     \le D_{(j)}\mathrm w_{n,(j)}-\eta_{(j)},& ~\mbox{if}~j=2,4; \\
      \ge D_{(j)}\mathrm w_{n,(j)}+\eta_{(j)},& ~\mbox{if}~j=3,5; \\
      \le \mathrm w_{n,(6)}-\eta_{(6)} ,& ~\mbox{if}~j=6,
\end{array}\right.
	\end{aligned}
\end{equation}
	where 
	$
	    (D_{(2)}~D_{(3)}~D_{(4)}~D_{(5)})=(\diag(\sign( s_{(2)}))~\diag(\sign( t_{(3)}))~1~-1).
	$
\end{definition}

When $s=0$, $t=+\infty$, $\mathrm w_n=1$ and $\Sigma_{n}$ is the identity matrix, the AREIC becomes the \textit{nonnegative elastic irrepresentable condition} (NEIC) as follows:
\begin{equation}
\label{NEIC}
C_{61}\Big(C_{11}+\frac{\lambda_n^{(2)}}{n}\Big)^{-1}\Big(\mathbf 1+\frac{2\lambda_n^{(2)}}{\lambda_n^{(1)}}\beta_{(1)}^*\Big)\le\mathbf 1 - \eta_{(6)},
\end{equation}
which was necessary to yield the variable selection consistency of nonnegative elastic net \cite{Zhao2014}. If, in addition to (\ref{NEIC}), $\lambda_n^{(2)}=0$, the NEIC then becomes the \textit{nonnegative irrepresentable condition} (NIC):
\begin{equation*}
% \label{NIC}
C_{61}C_{11}^{-1}\mathbf 1\le \mathbf 1 - \eta_{(6)},
\end{equation*}
which was a necessary condition to obtain the variable selection consistency of  the nonnegative lasso \cite{wu_nl}. Note that, NIC is a nonnegative version of the \textit{irrepresentable condition} (IC) for the variable selection consistency of the lasso \citep{zhao}:
\begin{equation*}
% \label{NIC}
|C_{61}C_{11}^{-1}\sign(\beta_{(1)}^*)|\le \mathbf 1 - \eta_{(6)}.
\end{equation*}
Although IC is a sufficient and necessary condition for the variable selection consistency of the lasso  \citep{zhao} while NIC is only a necessary condition, in the real world NIC is easier to be satisfied than IC since it does not require the absolute value on the left-hand side of the inequality.  
As a result AREIC is a natural general version of the previous necessary conditions NEIC and NIC for the variable selection consistency. 
Below we state the first main result of the paper. Its proof is given in Appendix \ref{proof_thm:VSC}.
\begin{theorem}
	\label{thm:VSC}
	Under AREIC and the conditions (\ref{con_0}) - (\ref{con_3}), the ARGEN possesses the variable selection consistency property (\ref{VSC}).
\end{theorem}

\subsection{Estimation Consistency}
Recall that an estimation method with target parameter $\beta^*$ has the property of estimation consistency if 
$$
\|\widehat{\beta} - \beta^*\|_2 \xrightarrow[n\to\infty]{\mathbb P} 0,
$$
where $\|\bullet\|_2$ denotes the Euclidean distance and $\xrightarrow[n\to\infty]{\mathbb P}$ is the convergence in probability. 
Besides the variable selection consistency, ARGEN admits estimation consistency, subject to the following conditions.
\begin{description}
\item[(i)] $\beta^*\in\mathcal I$. Let $p=p_n$, $q=q_n$ be non-decreasing as $n$ increases.
	\item[(ii)] $\mathrm w_n=(\mathrm w_{n,1},\ldots, \mathrm w_{n,p_n})$ with $\mathrm w_{n, 1},\ldots,\mathrm w_{n, p_n}>0$ and $\Sigma_n$ are given.
	\item[(iii)] 	
	Let $X_j$ be the $j$th column of $X$, which satisfies
	\begin{equation*}
		\max_{1\le j\le p_n}\frac{2(X_j'X_j+\lambda_n^{(2)}\Sigma_{n,jj})}{(1+\lambda_n^{(2)})\mathrm{w}_{n,j}^2} \leq 1,~\mbox{for all} ~ n\ge1.
	\end{equation*}
	\item[(iv)]
	$X$ satisfies the restricted eigenvalue (RE) condition, i.e. there exists a constant $\kappa>0$, such that for all $n\ge 1$ and all $\beta\in\mathcal I$ satisfying
		$$
		\sum_{j=4}^6\mathrm w_{n,(j)}'|\beta_{(j)}|\le 3\sum_{j=1}^3\mathrm w_{n,(j)}'|\beta_{(j)}|,
$$
we have
		\begin{equation*}
		2(\|X\beta\|_2^2+\lambda_n^{(2)}\beta'\Sigma_n \beta) \geq \kappa(1+\lambda_n^{(2)}) \|\diag(\mathrm w_{n})\beta\|_2^2.
		\end{equation*}
\item[(v)] $\lambda_n^{(1)}$, $\lambda_n^{(2)}$, $\mathrm w_{n}$, $p_n$ and $q_n$ satisfy
$$
\frac{q_n(\lambda_n^{(1)})^2}{(1+\lambda_n^{(2)})^2}\xrightarrow[n\to\infty]{}0~\mbox{and}~p_n\exp\bigg(-\frac{n}{8\sigma^2}\frac{\big(\lambda_n^{(1)}\big)^2}{1+\lambda_n^{(2)}}\bigg)\xrightarrow[n\to\infty]{}0,
$$
where $\sigma>0$ is the residual standard deviation of the ARGEN.
\end{description}
Below we state the estimation consistency of the ARGEN.

\begin{theorem}
	\label{thm:estimation_consistency}
	Consider a $q_n$-sparse instance of the ARGEN (\ref{model}). Let $X$ satisfy the conditions (i) - (iv) and let the regularization parameters $\lambda_{n}^{(1)}>0, \lambda_{n}^{(2)}\ge 0$, then the ARGEN solution $\widehat{\beta}:=\widehat{\beta}(\lambda_n^{(1)},\lambda_n^{(2)},\mathrm w_n, \Sigma_n)$ satisfies:	
	\begin{eqnarray}
	\label{17}
		&&\mathbb P\bigg(\|\diag(\mathrm w_{n})(\widehat{\beta} - \beta^*)\|_2^2 >\frac{9q_n(\lambda_n^{(1)})^2}{\kappa^2(1+\lambda_n^{(2)})^2}\bigg)\le 2p_n\exp\bigg(-\frac{n}{8\sigma^2}\frac{(\lambda_n^{(1)})^2}{1+\lambda_n^{(2)}}\bigg),\\
	\label{18}
		&&\mathbb P\bigg(\| \diag(\mathrm w_{n})(\widehat{\beta} - \beta^*) \|_1 > \frac{12q_n\lambda_n^{(1)}}{\kappa(1+\lambda_n^{(2)})}\bigg)\le 2p_n\exp\bigg(-\frac{n}{8\sigma^2}\frac{(\lambda_n^{(1)})^2}{1+\lambda_n^{(2)}}\bigg),
	\end{eqnarray}
	where $\sigma>0$ denotes the residual standard deviation of the ARGEN. 
	In addition if $(v)$ holds, we have
	\begin{equation}
	    \label{estimation_consis}
	    \|\widehat{\beta} - \beta^*\|_2\xrightarrow[n\to\infty]{\mathbb P}0.
	\end{equation}
\end{theorem}
\begin{proof}
The main idea to the proof is to transform the ARGEN problem into a rectangle-range lasso problem. Let
\begin{eqnarray*}
&&\widetilde{X}=\frac{\sqrt{2n}}{\sqrt{1+\lambda_n^{(2)}}}\begin{pmatrix}
     X\diag(\mathrm w_{n})^{-1}\\
     \sqrt{\lambda_n^{(2)}} \Sigma_n^{1/2}\diag(\mathrm w_{n})^{-1}
\end{pmatrix}_{(n+p)\times p },
~~~~
\widetilde{Y} = \begin{pmatrix}
     \sqrt{2n}Y\\
     0
\end{pmatrix}_{(n+p)\times 1 }, \\
&&\widetilde{\beta}^* = \sqrt{1+\lambda_n^{(2)}}\diag(\mathrm w_{n})\beta^*,
~~~~
\lambda_n = \frac{\lambda_n^{(1)}}{\sqrt{1+\lambda_n^{(2)}}},\\
&&
\widetilde{\mathcal I} =\prod_{i=1}^{p_n}\left[\sqrt{1+\lambda_n^{(2)}} \mathrm w_{n,i}s_i,\sqrt{1+\lambda_n^{(2)}} \mathrm w_{n,i}t_i\right],
\end{eqnarray*}
Then the ARGEN  (\ref{model}) can be written as the rectangle-range lasso:
\begin{equation}
\label{modified_lasso}
\widehat{\widetilde{\beta}}(\lambda_n) =\argmin_{\beta \in\widetilde{\mathcal I}}\left( \frac{1}{2n}\big\| \widetilde{Y} - \widetilde{X}\beta \big\|_2^2+\lambda_n |\beta|\right)=\sqrt{1+\lambda_n^{(2)}}\diag(\mathrm w_{n})\widehat{\beta}.
\end{equation}
In view of the conditions $(i)$-$(iv)$, all requirements of Corollary 2 in \cite{negahhan} are satisfied. Therefore, applying Corollary 2 in \cite{negahhan} to the lasso (\ref{modified_lasso}) yields the results. We point out that: $(1)$ Based on its proof, Corollary 2 in \cite{negahhan} works for rectangle-range lasso. $(2)$ There is a typo in the statement of Corollary 2 in \cite{negahhan}: the inequalities (34) in \cite{negahhan} should be corrected to
\begin{equation*}
  \|\widehat{\theta}_{\lambda_n} - \theta^* \|_2^2\le\frac{144\sigma^2}{\kappa_{\mathcal{L}}^2}\frac{s\log p}{n}~\mbox{and}~\|\widehat{\theta}_{\lambda_n} - \theta^* \|_1\le\frac{48\sigma}{\kappa_{\mathcal{L}}}s\sqrt{\frac{\log p}{n}}.
\end{equation*}
\end{proof}
If we assume $\mathrm w_n\nrightarrow 0$ as $n\to\infty$ in Theorem \ref{thm:estimation_consistency}, we  easily obtain the estimation consistency condition for the nonnegative lasso (see Proposition 1 in \cite{wu_nl}) and the nonnegative elastic net. Note that the estimation consistency of the nonnegative elastic net \cite{wu_nen} has not yet been derived, hence we state it below as a corollary of Theorem \ref{thm:estimation_consistency}. To obtain the corollary it suffices to observe $\sum_{i=1}^3\mathrm w_{n,(i)}'\mathrm w_{n,(i)}=q_n$
when $\mathrm w_{n,j}=1$ for all $n\ge1$,~$j=1,\ldots,p_n$.
\begin{corollary}
	\label{cor:estimation_consistency}
	Consider a $q_n$-sparse  nonnegative elastic net model. Assume:
\begin{description}
\item[(i)] $\beta^*\ge0$. $p_n,q_n$ are non-decreasing as $n$ increases.
	\item[(ii)] 	
	Let $X_j$ be the $j$th column of $X$ which satisfies
	\begin{equation*}
		\frac{2(X_j'X_j+\lambda_n^{(2)})}{1+\lambda_n^{(2)}} \leq 1,~\mbox{for all} ~ j = 1,\ldots,p.
	\end{equation*}
	\item[(iii)]
	There exists a constant $\kappa>0$, such that
	\begin{equation*}
		2(\|X\beta\|_2^2+\lambda_n^{(2)}\|\beta\|_2^2 )\geq \kappa(1+\lambda_n^{(2)}) \|\beta\|_2^2
		\end{equation*}
for all $\beta\ge0$ satisfying
		$$
		\sum_{j\in\{1,\ldots,p_n\}:~\beta_j^*=0}|\beta_{j}|\le 3\sum_{j\in\{1,\ldots,p_n\}:~\beta_j^*\ne0}|\beta_{j}|.
$$
\end{description}
	Let $\lambda_{n}^{(1)}>0, \lambda_{n}^{(2)}\ge 0$, then the nonnegative elastic net solution $\hat\beta$ verifies the following inequalities:
	\begin{eqnarray*}
		&&\mathbb P\bigg(\|\widehat{\beta} - \beta^* \|_2^2 \leq\frac{9q_n(\lambda_n^{(1)})^2}{\kappa^2(1+\lambda_n^{(2)})^2}\bigg)\ge 1-2p_n \exp\bigg(-\frac{n(\lambda_{n}^{(1)})^2}{8\sigma^2(1+\lambda_{n}^{(2)})}\bigg),\nonumber\\
		&&\mathbb P\bigg(\| \widehat{\beta} - \beta^* \|_1\leq \frac{12q_n\lambda_n^{(1)}}{\kappa(1+\lambda_n^{(2)})}\bigg)\ge 1-2p_n \exp\bigg(-\frac{n(\lambda_{n}^{(1)})^2}{8\sigma^2(1+\lambda_{n}^{(2)})}\bigg).
	\end{eqnarray*}
\end{corollary}
As a consequence of  Corollary \ref{cor:estimation_consistency}, $\widehat\beta$ is consistent if 
$$
\frac{q_n(\lambda_n^{(1)})^2}{(1+\lambda_n^{(2)})^2}\xrightarrow[n\to\infty]{}0~\mbox{and}~p_n \exp\bigg(-\frac{n(\lambda_{n}^{(1)})^2}{8\sigma^2(1+\lambda_{n}^{(2)})}\bigg)\xrightarrow[n\to\infty]{}0.
$$
If we take $\lambda_n^{(2)}=0$ and $\lambda_n^{(1)}=4\sigma\sqrt{\log p_n/n}$ in Corollary \ref{cor:estimation_consistency}, we obtain the nonnegative lasso's tail probability control as in Proposition 1 in \cite{wu_nl}. If we further assume $\beta^*\in\mathbb R$ in  Corollary \ref{cor:estimation_consistency}, we derive the tail bounds for the lasso  (see Corollary 2 in \cite{negahhan}). 

\subsection{Limiting Distributions of ARGEN Estimators}
We now study the asymptotic behavior in distribution of the ARGEN estimators, as $n\to\infty$. Again we can make use of the transformation of ARGEN to the rectangle-range lasso model (\ref{modified_lasso}), since the limiting distributions of the lasso regression estimators have been studied in \cite{knight2000asymptotics}. Observe that (\ref{modified_lasso}) is equivalent to
\begin{equation}
\label{modified_lasso_1}
\widehat{\widetilde{\beta}}(\lambda_n) =\argmin_{\beta \in\widetilde{\mathcal I}}\left(\big\| \breve{Y} - \breve{X}\beta \big\|_2^2+\lambda_n |\beta|\right),
\end{equation}
where $\breve X=\widetilde X/\sqrt{2n}$ and $\breve Y=\widetilde Y/\sqrt{2n}$.  (\ref{modified_lasso_1}) is then the type of lasso studied in \cite{knight2000asymptotics}.  Assume that the row vectors of $\breve X$, denoted by $\breve X^{(i)}$, $i=1,\ldots,n$, satisfy
\begin{equation}
\label{lim_matrix}
	\frac{1}{n}\sum_{i=1}^n\breve X^{(i)}{\breve X^{(i)}}{}'\xrightarrow[n\to\infty]{}M,
	\end{equation}
	where $M$ is a nonsigular nonnegative definite matrix and
\begin{equation}
\label{lim_value}
		\frac{1}{n}\max\limits_{1\le i\le n}{\breve X^{(i)}}{}'\breve X^{(i)}\xrightarrow[n\to\infty]{}0.
	\end{equation}
It follows from Theorem 2 in \cite{knight2000asymptotics} that the ARGEN estimator $\widehat\beta$ has the following asymptotic behavior in distribution.
\begin{theorem}
\label{thm:convergence_law}
Assume $\lim_{n\to\infty}p_n=p$,  $\lim_{n\to\infty}\lambda_n^{(2)}=\lambda^{(2)}$ and $\lim_{n\to\infty}\mathrm w_n=\mathrm w=(\mathrm w_1,\ldots,\mathrm w_p)$. Let $X$, $\mathrm w_n$ and $\Sigma_n$ satisfy (\ref{lim_matrix}) and (\ref{lim_value}). Also assume 
$$
\frac{\lambda_n^{(1)}}{\sqrt{n(1+\lambda_n^{(2)})}}\xrightarrow[n\to\infty]{}\lambda_0\ge0.
$$
Then
$$
\sqrt{n}(\widehat\beta-\beta^*)\xrightarrow[n\to\infty]{\mbox{law}}\argmin_{u\in \mathcal I}(V(u)),
$$
where $\xrightarrow[n\to\infty]{\mbox{law}}$ denotes the convergence in distribution; $V(u)$ is a Gaussian random variable given as
$$
V(u)=-2u'G+u'Mu+\lambda_0\sum_{j=1}^p\left(u_j\sign(\beta_j^*)\mathds{1}(\beta_j^*\ne0)+|u_j|\mathds{1}(\beta_j^*=0)\right).
$$
In the above expression of $V(u)$, $G\sim \mathcal N(0,\sigma^2M)$, $u_j$ denotes the $j$th coordinate of $u$ and $\mathds{1}$ is the indicator function.
\end{theorem}
 Theorem \ref{thm:convergence_law} includes the asymptotic behaviors of the elastic net and nonnegative elastic net estimators as its particular examples. As another particular example, when $\lambda_0=0$ and $p=1$ (then $M$ is a single value and $\mathcal I=[s_1,t_1]$), by the fact that $V$ is convex, we obtain
 $$
 \argmin_{u\in \mathcal I}(V(u))=\left\{
 \begin{array}{ll}
 M^{-1}G&~\mbox{if}~M^{-1}G\in[s_1,t_1];\\
s_1&~\mbox{if}~M^{-1}G<s_1;\\
t_1&~\mbox{if}~M^{-1}G>t_1.
 \end{array}
 \right.
 $$
 In the above example, if $p\ge2$ and $\mathcal I\ne \mathbb R^p$, $\argmin_{u\in \mathcal I}(V(u))$ has no simple explicit expression. Note that $\argmin_{u\in \mathcal I}(V(u))$ belongs to some quadratic programming problem. In the next section we provide a multiplicative updates numerical algorithm to solve the ARGEN. This algorithm may be further applied to simulate $\argmin_{u\in \mathcal I}(V(u))$ numerically.
 
\section{MU-QP-RR-W-\texorpdfstring{$l_1$}{TEXT} Algorithm for Solving ARGEN}
\label{sec:numerical_scheme}
In this section we provide a solution of ARGEN by using an extensive multiplicative updates algorithm. Given $\mathrm w_n,~\Sigma_n$ and $\lambda_n^{(1)},\lambda_n^{(2)}\ge0$, the ARGEN in (\ref{model}) can be expressed as the following equivalent problem:
\begin{equation}
\label{7}
	\begin{cases}
	\text{minimize }  F_1(\beta) = \beta' \big(X'X+\lambda_n^{(2)}\Sigma_n\big)\beta - 2(X'Y)'\beta + \lambda_n^{(1)} \mathrm w_{n}'\left|\beta\right|, \\
	\text{subject to } \beta \in [ s, t].
	\end{cases}
\end{equation}
To simplify the problem, we rewrite it by taking 
$$
\left\{
\begin{array}{ll}
&v = \beta - s,\\
&A = 2\big(X'X+\lambda_n^{(2)}\Sigma_n\big),\\
&b = A s- 2X'Y,\\
&d = \lambda_n^{(1)}\mathrm w_n,\\
& l =  t- s,\\
& v^0 =  s^-:=(\max\{0,-s_1\}~\ldots~\max\{0,-s_p\})'\ge0
\end{array}\right.,
$$
and obtain an equivalent problem of (\ref{7}), that is,
\begin{equation}
\label{7'''}
	\begin{cases}
	\text{minimize } F(v) = v'Av+b'v+d'|v-v^0|, \\
	\text{subject to } v \in [0, l].
	\end{cases}
\end{equation}
This is obtained by arguing
$
|v+ s|-|v-v^0|= s^+:=\left(\max\{0,s_1\}~\ldots~\max\{0,s_p\}\right)'$
and omitting the constant terms
$
d' s^++ (3/2)s'A s-b' s.
$ 
Here the matrix $A$ is symmetric positive semi-definite. The problem (\ref{7'''}) is a quadratic programming problem but contains an item of $d'|v-v^0|$.

Sha et al. \cite{sha} derived the multiplicative updates for solving the nonnegative quadratic
programming problem. The algorithm has been shown to have a simple closed-form, and a rapid convergence rate. For our problem (\ref{7'''}), however, it contains the absolute values and lower and upper limits of the optimization variables, so that direct application of the algorithm  in \cite{sha} is impractical. Therefore, we propose a new iterative algorithm to  solve (\ref{7'''})  and call it \textit{multiplicative updates for solving  quadratic programming with rectangle range and weighted $l_1$ regularizer} (abbreviated to  MU-QP-RR-W-$l_1$).

Let us formulate a more general problem that can be solved by MU-QP-RR-W-$l_1$:
\begin{equation}
\label{4}
\begin{cases}
\mbox{minimize} ~F(v) = \frac{1}{2} v'Av + b' v+d'|v-v^0|, \\
\mbox{subject to}~v\in[0,  l]. 
\end{cases}
\end{equation}
Here $v,b,d,v^0,l$ are column vectors of dimension $p$, where elements of $d, v^0$ are nonnegative and elements of $l$ are positive. The matrix $ A=(A_{ij})_{1\le i,j\le p}$ is positive semi-definite. In fact, the nonnegative quadratic programming   (see e.g. Equation (5) in \cite{sha} or (20) in \cite{wu_nen}) is a special case of (\ref{4}), where we take the elements of $d, v^0$ to be $0$ and elements of $l$ to be infinity.

Let us further adopt the following notations. For $i,j\in\{1,\ldots,p\}$, we define the positive part and negative part of $A_{ij}$ by
\begin{equation*}
% \label{5}
	A_{ij}^{+} := \max\left\{0,A_{ij}\right\}~
	\text{  and  }~
	A_{ij}^{-} := \max\left\{0,-A_{ij}\right\}.
\end{equation*}
Then denote the positive part and negative part of the matrix $A$ by
$$
A^+:=\big(A_{ij}^+\big)_{1\le i,j\le p}~\mbox{and}~A^-:=\big(A_{ij}^{-}\big)_{1\le i,j\le p}.
$$
It follows that 
$A=A^+-A^-~\mbox{and}~|A|:=\left(|A_{ij}|\right)_{1\le i,j\le p}=A^++A^-$.  
Let 
$
a_i(v) := (A^+ v)_i~\mbox{and}~c_i(v) := (A^- v)_i
$, we then present the MU-QP-RR-W-$l_1$ algorithm in pseudocode in Algorithm \ref{alg:MU}.

\begin{algorithm}[!tbh]
\caption{MU-QP-RR-W-$l_1$} \label{alg:MU}
\LinesNumbered 
\KwIn{ $A,b,d,v^0,l$.}
\textbf{Initialization:} $v^{(1)}>0$; $v^{(0)}\longleftarrow0$; $m\longleftarrow1$\;

\While{$v^{(m-1)} \neq v^{(m)}$}{
\For{$i=1,\ldots,p$}{
\begin{align}
    \label{eq:noboundIter}
   & r_1 \longleftarrow  v_i^{(m)}\Big(\frac{-(b_i+d_i) + \sqrt{(b_i+d_i)^2 + 4a_i(v^{(m)}) c_i(v^{(m)})}}{2a_i(v^{(m)})}\Big); \nonumber\\
   & r_2  \longleftarrow  v_i^{(m)}\Big(\frac{-(b_i-d_i) + \sqrt{(b_i-d_i)^2 + 4a_i(v^{(m)}) c_i(v^{(m)})}}{2a_i(v^{(m)})}\Big); \nonumber\\
   & v^{(m+1)}_i\longleftarrow\left\{
\begin{array}{cc}
   \min\{r_1,l_i\} & \mbox{if~} r_1>v_i^0; \\
   \min\{r_2,l_i\} & \mbox{if~} r_2<v_i^0; \\
    \min\{v_i^0,l_i\} & \mbox{otherwise};
\end{array}\right.    
    \end{align}
    }
    $m\longleftarrow m+1$\;
    }
\KwOut{$v^{(m)}$.}
\end{algorithm}

We point out that the conditions $r_1>v_i^0$ and $r_2<v_i^0$ in (\ref{eq:noboundIter}) are mutually exclusive when $v_i^{(m)}>0$. This is because, on one hand, $r_1>v_i^0$ is equivalent to
    \begin{equation}
    \label{ineq_case_1}
    \frac{2a_i(v)v_i^0}{v_i}+(b_i+d_i)<0
    \quad\quad\mbox{or}\quad\quad  \left\{
    \begin{array}{ll}
         & \frac{2a_i(v)v_i^0}{v_i}+(b_i+d_i)\ge0, \\
         & a_i(v)\left(\frac{v_i^0}{v_i}\right)^2+(b_i+d_i)\frac{v_i^0}{v_i}-c_i(v)<0.
    \end{array}\right.
    \end{equation}
On the other hand, $r_2<v_i^0$ is equivalent to
 \begin{equation}
    \label{ineq_case_2}
    \left\{
    \begin{array}{ll}
         & \frac{2a_i(v)v_i^0}{v_i}+(b_i-d_i)\ge0, \\
         & a_i(v)\left(\frac{v_i^0}{v_i}\right)^2+(b_i-d_i)\frac{v_i^0}{v_i}-c_i(v)>0.
    \end{array}\right.
   \end{equation}
Since $d_i\ge0$, it is obvious that (\ref{ineq_case_1}) and (\ref{ineq_case_2}) are mutually exclusive.

A special case of the algorithm is that when $d_i=0$ and $l_i=+\infty$ for $i=1,\ldots,p$, (\ref{eq:noboundIter}) becomes
$$
v_i\longleftarrow v_i\Big(\frac{-b_i + \sqrt{b_i^2 + 4a_i(v) c_i(v)}}{2a_i(v)}\Big),~\mbox{for}~i=1,\ldots,p,
$$
which reduces to the one for nonnegative quadratic programming  (see e.g. (7) - (12) in \cite{sha} or (21) - (22) in \cite{wu_nen}).

The MU-QP-RR-W-$l_1$ converges monotonically to the global (rectangle area) minimum of the objective function $F(\bullet)$ in (\ref{4}). This is summarized in the theorem below:
\begin{theorem}
\label{thm:conistency_algo}
 Let $F(\bullet),~A,~b,~d,~v^0,~ l$ be given as in the problem (\ref{4}). Define an auxiliary function $G(\bullet,\bullet)$ by: for $u,v\in[0, l]$,
 \begin{equation}
 \label{def::auxiliary function}
     G(u,v):=\frac{1}{2}\sum_{1\le i,j\le p}\Big(\frac{A^+_{ij}u_i^2v_j}{v_i}-A^-_{ij}v_iv_j\big(1+\log\frac{u_iu_j}{v_iv_j}\big)\Big)+\sum_{i=1}^p (b_iu_i+d_i|u_i-v^0_i|).
 \end{equation}
For any positive-valued vector $v\in[0, l]$, pick a vector $U(v)\in \argmin_{ u\in[0, l]}G(u, v)
$. Then $U(v)$ satisfies the following: 
\begin{description}
    \item[(i)] For any $v\in[0, l]$,
    \begin{eqnarray}
        \label{U_ineq}
        &&F(U(v))\leq F(v),\\
        \label{U_eq}
        &&F(U(v))= F(v)~\mbox{if and only if}~U(v)=v.
    \end{eqnarray}
    \item[(ii)] For each $v\in[0, l]$, $U(v)$ is the updated value of $v$, presented in the form of (\ref{eq:noboundIter}). 
\end{description}
\end{theorem}
The approach we used to prove Theorem \ref{thm:conistency_algo} is similar to the ones used in  Expectation-Maximization algorithm \citep{dempster1977maximum}, nonnegative matrix factorization \citep{lee2001algorithms}, and Multiplicative Updates for Nonnegative Quadratic Programming \citep{sha,sha2007}. More specifically, 
the proof proceeds in two steps. First, we establish an auxiliary function $G(\bullet,\bullet)$ in (\ref{def::auxiliary function}) to show that the MU-QP-RR-W-$l_1$ monotonically decreases the value of the objective function $F(\bullet)$ in (\ref{4}). Then, we show that the iteratively updates  (\ref{eq:noboundIter}) in Algorithm \ref{alg:MU} converge to the global minimum. The complete proof of Theorem \ref{thm:conistency_algo} is provided in Appendix \ref{sec:proof_thm_conistency_algo}.

\section{Hyper-parameter Optimization}
\label{Hyper_parameter_tuning}
ARGEN is a family including many well-known linear models with constraints. For instance, in Table \ref{model_set}, the ARLS, ARL, ARR, and AREN correspond to the arbitrary rectangle-range least squares, lasso, ridge, and elastic net, respectively. Besides, based on the choice of parameters we can propose some other new methods including ARGL, ARGR, ARLEN and ARREN, which are applicable to more complicated problems, and usually perform better than the free-range regression coefficients models.

However, many of the methods in Table \ref{model_set} involve dealing with very high-dimensional hyper-parameter space, thus grid search method for tuning parameters can be very computationally expensive. Other tuning approaches such as Bayesian optimization and gradient-based optimization, developed to obtain better results in fewer evaluations, when applied to our case, however, are also very costly because they easily search around local minimum when the complexity of the surface is relatively high. Because discovering better tuning methods is not a focus in this paper, it is a potential direction of our future research, thus we simply use random search ($N_{calls}$ trials) to avoid costly searching on the entire grid. Following the convention in \cite{zouhas} and  \cite{tibshirani}, we use the mean-squared
error (MSE) as the score function, that is,
\begin{equation}
\label{me}
    MSE=\mathbb E[X\widehat{\beta}-X\beta^*]^2=\mathbb E\big[(\widehat{\beta}-\beta^*)'X'X(\widehat{\beta}-\beta^*)\big].
\end{equation}

\begin{table}[H]
\caption{Particular examples of ARGEN methods and their parameter setting. $1/p$ is the value of each element of $\mathrm w_n$. $I$ represents the identity matrix in $\mathbb{R}^{p\times p}$. Parameters with no value assigned are hyper-parameters. AR stands for arbitrary rectangle-range.}
\label{model_set}
\centering
\begin{tabular}{lccccc}
\toprule 
 Method& Abbreviation& $\lambda_n^{(1)}$& $\lambda_n^{(2)}$ & $\mathrm w_n$ & $\mathrm\Sigma_n$ \tabularnewline
\midrule 
AR least squares& ARLS & $0$ & $0$ & $1/p$ & $I$\tabularnewline
AR lasso& ARL &  & $0$ & $1/p$ & $I$\tabularnewline
AR generalized lasso& ARGL &  & $0$ &  & $I$\tabularnewline
AR ridge &ARR & $0$ &  & $1/p$ & $I$\tabularnewline
AR generalized ridge& ARGR & $0$ &  & $1/p$ & \tabularnewline
AR elastic net& AREN &  &  & $1/p$ & $I$\tabularnewline
AR lasso generalized elastic net& ARLEN &  &  &  & $I$\tabularnewline
AR ridge generalized elastic net& ARREN &  &  & $1/p$ & \tabularnewline
AR generalized elastic net& ARGEN &  &  &  & \tabularnewline
\bottomrule 
\end{tabular}
\end{table}

To further speed up the tuning process, we select the following values for each of the  potential hyper-parameters to tune on. $\lambda_n^{(1)},~\lambda_n^{(2)}$ take integer values in the sets $\Lambda^{(1)}$ and $\Lambda^{(2)}$, respectively, where for each $i=1,2$,
$
\Lambda^{(i)}=\big\{ 0,~1,\ldots,\lambda_{up}^{(i)}\big\}~\mbox{for some}~\lambda_{up}^{(i)}\in \mathbb{Z_+}.
$
The weight vector $\mathrm w_n$ takes values in
$$
 W=\left\{\frac{(w_1~\ldots~w_p)'}{\sum_{i=1}^p w_i}: w_1,\ldots,w_p\in\left\{0,\cdots,w_{up} \right\}\right\},~\mbox{for some}~w_{up}\in\mathbb{Z^+}.
$$ 
The matrix $\mathrm\Sigma_n$ can be decomposed to $\Sigma_n=PDP'$ with orthogonal matrix $P$ and nonnegative diagonal matrix $D$. Therefore, values of $\mathrm\Sigma_n$ are considered in
$$
\Sigma=\left\{PDP': D=\diag(d_1,\cdots,d_p), ~ d_1,\ldots,d_p\in\left\{0,\cdots,d_{up}\right\} \right\},
$$
for some $d_{up}\in\mathbb{Z^+}$ and orthogonal matrix $P$. The cardinality of $\Lambda^{(1)}$, $\Lambda^{(2)}$, $W$, and $\Sigma$ are $\lambda_{up}^{(1)}+1$, $\lambda_{up}^{(2)}+1$, $(w_{up}+1)^p$, and $(d_{up}+1)^p$, respectively, thus we obtain the total values on the parameter grid for each method listed in Table \ref{model_tune}.

To improve the performance of the methods, we can: (1) randomly choose more trials, that is, increase  $N_{calls}$, to cover more values on the grid, since theoretically searching on the whole grid will give the best result; (2) increase the values of $w_{up}$ and $d_{up}$, but at the same time $N_{calls}$ also needs to be increased since higher values of $w_{up}$ and $d_{up}$ will exponentially enlarge the grid.

\begin{table}[H]
\caption{Tuning grid for different methods.}
\label{model_tune}
\centering
\begin{tabular}{lcc}
\toprule 
 Methods& Tuning grid & Total number of values on grid\tabularnewline
\midrule 
 ARLS & None&0 \tabularnewline
 ARL& $\Lambda^{(1)}$ &$\lambda_{up}^{(1)}+1$  \tabularnewline
ARGL& $\Lambda^{(1)}\times W$& $(\lambda_{up}^{(1)}+1)(w_{up}+1)^p$ \tabularnewline
ARR& $\Lambda^{(2)}$&$\lambda_{up}^{(2)}+1$ \tabularnewline
ARGR& $\Lambda^{(2)}\times\Sigma$& $(\lambda_{up}^{(2)}+1)(d_{up}+1)^p$\tabularnewline
AREN& $\Lambda^{(1)}\times\Lambda^{(2)}$&$(\lambda_{up}^{(1)}+1)(\lambda_{up}^{(2)}+1)$ \tabularnewline
ARLEN& $\Lambda^{(1)}\times\Lambda^{(2)}\times W$&$(\lambda_{up}^{(1)}+1)(\lambda_{up}^{(2)}+1)(w_{up}+1)^p$ \tabularnewline
ARREN& $\Lambda^{(1)}\times\Lambda^{(2)}\times\Sigma$&$(\lambda_{up}^{(1)}+1)(\lambda_{up}^{(2)}+1)(d_{up}+1)^p$ \tabularnewline
 ARGEN& $\Lambda^{(1)}\times\Lambda^{(2)}\times W\times\Sigma$&$(\lambda_{up}^{(1)}+1)(\lambda_{up}^{(2)}+1)(w_{up}+1)^p(d_{up}+1)^p$ \tabularnewline
\bottomrule 
\end{tabular}
\end{table}

\section{Simulations}
\label{sec:simulation}

\subsection{Signal Recovery}

The purpose of the  signal recovery examples below is to explore the best possible performance of ARGEN, and to show ARGEN's ability to \enquote{reduce noise} and deal with  high-dimensional ($p\gg n$) sparse signals.

First, we conduct the same problem as in \cite{mohammadi} to compare our results with theirs. In the following, we briefly outline the problem. A sparse signal $\beta^*\in\mathbb R^{4096}$ with 160 spikes that have amplitude $1$ is generated and plotted in Figure $\ref{signal_simple}$ (top), which is the true regression coefficient vector. The design matrix $X\in\mathbb R^{1024\times4096}$ is generated with each entry sampled from i.i.d. standard normal distribution and each pair of rows orthogonalized. The response vector $Y\in\mathbb R^{1024}$ is then generated through $Y=X\beta^*+\epsilon$, where $\epsilon\in R^{1024}$ is a vector of i.i.d Gaussian noise with zero mean and variance 0.1. The lower and upper bounds of $\widehat \beta$ are $-1$ and $1$, respectively. Given $\lambda_n^{(1)}=10,  \lambda_n^{(2)}=0,$ and $w_i=0 $ if $\beta_i\neq 0 $ for $i=1,\ldots, 4096$, we obtain the recovery signal $\widehat \beta$ and its difference from the true signal in the middle and bottom plots in Figure $\ref{signal_simple}$. As a result, ARGEN achieves a lower MSE of $0.00069$ compared with the MSE of $0.00273$ in \cite{mohammadi}. 

To show that ARGEN can deal with more complicated problems, we take another signal recovering problem as example. We follow the same settings as in the previous problem, but this time replace the amplitude of each spike by a random value generated from a uniform distribution over $[0, 1)$. The corresponding true and recovery signals and their difference are plotted in Figure $\ref{signal_arbitrary}$. The MSE obtained by ARGEN is $0.00166$.

\begin{figure}[!tbh]
\centering
\begin{minipage}[b]{0.49\textwidth}
\includegraphics[width=\textwidth]{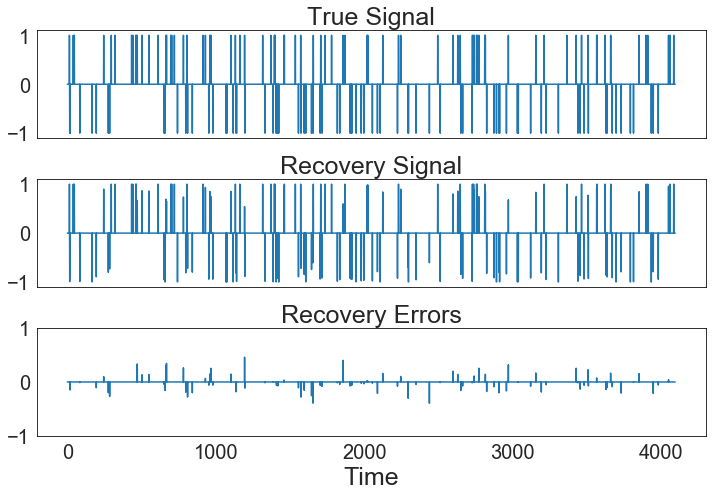}
\caption{Signal recovery with equal-length spikes. MSE $= 0.00069$.}
\label{signal_simple}
\end{minipage}
  \hfill
  \begin{minipage}[b]{0.49\textwidth}
\includegraphics[width=\textwidth]{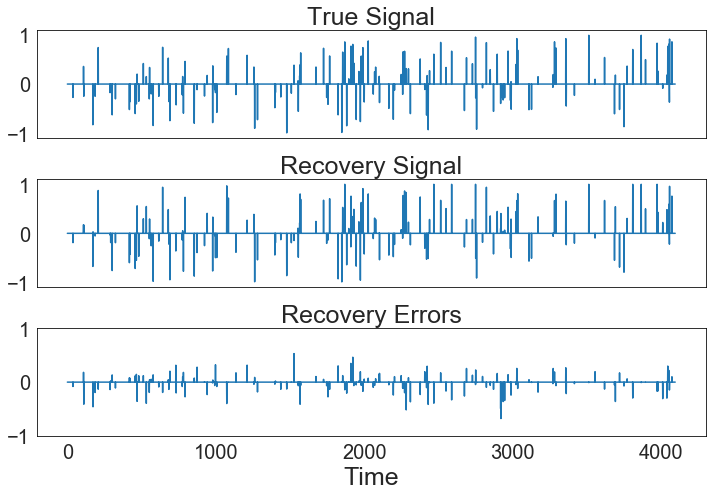}
\caption{Signal recovery with arbitrary-length spikes. MSE $= 0.00166$.}
\label{signal_arbitrary}
\end{minipage}
\end{figure} 

\subsection{Methods Comparison}
In this section, we compare the performances of the methods listed in Table \ref{model_set}. We adopt the following setup to tune the four hyper-parameters: $\lambda_{up}^{(1)}=100,~ \lambda_{up}^{(2)}=100,~w_{up}=2,~d_{up}=2$. Considering the computational cost, the number of random values ($N_{calls}$) on grid to try is $100$ for ARL and ARR, $500$ for AREN, $1280$ for ARGL and ARGR, $2560$ for ARLEN and ARREN, and $6554$ for ARGEN.   

We conduct eight examples to test the performance of each method. In each example, we simulate 50 datasets from $
Y=X\beta^*+\epsilon, \epsilon\sim \mathcal N(0,\sigma^2),
$ and each of the data sets consists of independent training, validation, and testing sets. We use the training set to fit models. Parameters are tuned on the validation set. The test error, measured by the MSE (\ref{me}), will be computed on the testing set. In the following, we outline these examples.

In Example 1, let $\beta^*=(3~1.5~0~0~2~0~0~0)'$, $p=8$, $\sigma=3$, and the pairwise correlation between $X_i$ and $X_j$ be $0.5^{|i-j|}$ for all $i,j$. We use $20$ observations for training, $20$ for validation, and $200$ for testing.

Example 2 is the same as Example 1, except that each entry of $\beta^*$ is replaced with $0.85$.

In Example 3, let $\sigma=15$, $p=40$,
$$
\beta^*=(\underbrace{0~\cdots~0}_\textrm{10~times}~\underbrace{2~\cdots~2}_\textrm{10~times}~\underbrace{0~\cdots~0}_\textrm{10~times}~\underbrace{2~\cdots~2}_\textrm{10~times})',
$$ 
and the pairwise correlation between $X_i$ and $X_j$ be $0.5$ for all $i,j$. We use $100$ observations for training, $100$ for validation, and $400$ for testing.

In Example 4, let $\sigma=15$, $p=15$,
$$
\beta^*=(\underbrace{3~\cdots~3}_\textrm{6 times}~\underbrace{0~\cdots~0}_\textrm{9 times})'.
$$
Let the design matrix $X$ be generated as:
\begin{eqnarray*}
&&x_i=Z_1+\epsilon_i^x, ~~Z_1\sim \mathcal N(0,1), ~~i=1,2,\\
&&x_i=Z_2+\epsilon_i^x, ~~Z_2\sim \mathcal N(0,1), ~~i=3,4,\\
&&x_i=Z_3+\epsilon_i^x, ~~Z_3\sim \mathcal N(0,1), ~~i=5,6,\\
&&\mbox{where}~\epsilon_i^x~ \mbox{are i.i.d.}~ \mathcal N(0,0.01),~ i=1,\cdots,6. \\
&&x_i\sim N(0,1), ~~x_i \mbox{~are i.i.d.}, ~~i=7,\cdots,15.
\end{eqnarray*}
We use $40$ observations for training, $40$ for validation, and $100$ for testing.

Example 5 is the same as Example 1, except that  $\beta^*=(-3~-1.5~0~0~2~0~0~0)'$ and $\beta_i\geq -1000$ for all $i$. 

Example 6 is the same as Example 1, except that each entry of  $\beta^*$ is replaced with a random generated number in $[-5,5]$ and this values is used  for all the 50 data sets. Besides, we restrict $\beta_i\in [-5,5]$ for all $i$.

Example 7 is the same as Example 1, but uses $\beta^*=(-6~-8~0~0~7~0~0~0)'$ and restricts $\beta_i\in [-5,5]$ for all $i$.

Example 8 is the same as Example 4, but uses $5$ observations for training, $5$ for validation, and $50$ for testing. Beside, we restrict $\beta_i\geq -1000$ for all $i$ and use $$\beta^*=(\underbrace{-3~\cdots~-3}_\textrm{6 times}~\underbrace{0~\cdots~0}_\textrm{9 times})'.$$
 
The first three examples above are from \cite{zouhas} and \cite{tibshirani}, which are originally constructed for lasso. The fourth example is similar to that in \cite{zouhas}, which creates a grouped variable situation. None of the first four examples, however, requires constraints on lower or upper bound for the coefficients. To show and test that ARGEN is applicable to more general and complicated problems, we add four more examples, which are Examples 5 to 8. In each of  Examples 5, 6, and 8, constraints are added and include the true coefficients. In Example 7, we provide a case when the true coefficients are out of the interval constraints. The values $1000$ and $5$ were chosen arbitrarily to illustrate the model's ability to work with constrained coefficients. Moreover, another purpose of introducing the last example is to test model performance on high-dimensional $(p\geq n)$ scenarios.

\begin{table}[!tbh]
\caption{Median MSE and the corresponding standard error (given in the parentheses) over 50 replications for each method and example.}
\label{model_result}
\centering
\begin{tabular}{lcccc}
\toprule 
\multirow{2}{*}{Models} & Example 1& Example 2 &Example 3 & Example 4  \tabularnewline
\cmidrule{2-5}
&MSE (SE)&MSE (SE)&MSE (SE)&MSE (SE)\tabularnewline
\midrule
ARLS & $2.28~ (0.31)$ & $3.30~(0.27)$ & $43.34~(1.53)$ &$159.47~(6.23)$\tabularnewline
ARL & $1.55~ (0.26)$ & $2.61~(0.21)$ & $41.72~(1.43)$ &$130.98~(3.91)$\tabularnewline
ARR & $1.73~ (0.25)$ & $1.31~(0.15)$ & $19.55~(0.72)$ &$111.63~(2.23)$\tabularnewline
AREN & $1.53~ (0.29)$ & $1.43~(0.16)$ & $19.58~(0.73)$ &$109.44~(2.28)$\tabularnewline
ARGL & $0.54~ (0.17)$ & $1.39~(0.13)$ & $42.87~(1.52)$ &$155.69~(5.70)$\tabularnewline
ARGR & $0.71~ (0.13)$ & $0.92~(0.09)$ & $39.12~(1.36)$ &$138.12~(3.81)$\tabularnewline
ARLEN & $0.62~ (0.19)$ & $1.22~(0.14)$ & $19.45~(0.72)$ &$111.47~(2.25)$\tabularnewline
ARREN & $0.49~ (0.14)$ & $1.01~(0.11)$ & $37.50~(1.31)$ &$124.11~(2.91)$\tabularnewline
ARGEN & $0.20~ (0.11)$ & $0.65~(0.11)$ & $38.25~(1.38)$ &$136.00~(3.41)$\tabularnewline
\bottomrule 
\toprule 
\multirow{2}{*}{Models} & Example 5& Example 6 &Example 7 & Example 8\tabularnewline
\cmidrule{2-5}
&MSE (SE)&MSE (SE)&MSE (SE)&MSE (SE)\tabularnewline
\midrule
ARLS & $5.81~ (0.45)$ & $3.92~(0.34)$ & $18.52~(0.82)$ &$10^7~(5\times10^5)$\tabularnewline
ARL & $2.61~ (0.31)$ & $3.38~(0.29)$ & $15.42~(0.48)$ &$117.02~(7.53)$\tabularnewline
AREN & $2.63~ (0.30)$ & $4.58~(0.71)$ & $15.72~(0.48)$ &$96.75~(5.36)$\tabularnewline
ARR & $3.04~ (0.30)$ & $3.98~(0.40)$ & $17.10~(0.66)$ &$98.49~(5.62)$\tabularnewline
ARGL & $1.14~ (0.20)$ & $1.78~(0.22)$ & $15.53~(0.35)$ &$357.39~(65.14)$\tabularnewline
ARGR & $1.51~ (0.24)$ & $1.30~(0.22)$ & $14.80~(0.42)$ &$189.69~(30.35)$\tabularnewline
ARLEN & $1.19~ (0.23)$ & $2.36~(0.25)$ & $15.52~(0.33)$ &$94.43~(8.39)$\tabularnewline
ARREN & $1.15~ (0.26)$ & $1.70~(0.23)$ & $14.67~(0.32)$ &$94.73~(13.07)$\tabularnewline
ARGEN & $0.72~ (0.14)$ & $1.25~(0.19)$ & $14.41~(0.31)$ &$152.57~(91.40)$\tabularnewline
\bottomrule 
\end{tabular}
\end{table}

Table \ref{model_result} summarizes the Median MSE and its corresponding standard error over 50 data sets using each method in Table \ref{model_set} for each of the above eight examples. The MSEs of examples with different $\sigma$ are not comparable because they are simulated with different noise variances. Some of our examples do, however, share a similar simulation process and their MSEs are at the same level. For instance, Examples 1 and 5 are similar except that Example 5 has a lower limit of $-1000$ on the coefficients, whereas Example 1 has no limit. As a result of these lower constraints, Example 5 tends to force coefficients above the lower limit, resulting in relatively higher MSEs than Example 1. In addition to cross-example comparisons, it would make more sense to compare the MSEs across methods for each example. The overall performance of methods consisting of more parameters is better than those with fewer parameters. More specifically, the ARGL, ARGR, ARLEN, ARREN, and ARGEN are, in most cases, outperforming the ARLS, ARL, ARR, and AREN. For instance, ARGEN, the most complicated method that includes all four hyper-parameters, performs best in Examples 1, 2, 5, 6, and 7. ARLEN and ARREN, the second from the top regard to complicity, provide second high accuracy in Examples 1, 2, 3, 4, 5, 7, and 8. ARGL and ARGR are at the third level of performance in Example 1, 2, 5, 6, and 7. However, in Table \ref{model_tune}, performances are not always increasing as the model gets more complicated. This is because the ratio of values searched ($N_{calls}$) to the total number of values on the grid is not the same for all the methods, due to the exponential increase of the size of the grid as more hyper-parameters are included. It is also because we keep the same $N_{calls}$ in each method for all the examples, which, in fact, have different dimensionality. Therefore, the performance of methods like ARLEN, ARREN, and ARGEN is worse than expected in some of the examples. 

\section{Real World Application - S\&P 500 Index Tracking}
\label{sec:real_world}
\subsection{Outline}
Index tracking is passive management that  replicates the return of a market index (e.g., S\&P 500  in  New York  and FTSE 100  in  London) by constructing an equity portfolio that contains only a subset of the index constituents to reduce the transaction and management costs \citep{connor,franks,jacobs,jobst,larsen,lobo,toy}.

In this section, we show how ARGEN applies to index tracking, an asset allocation \citep{markowitz} problem with allocation constraints, in the financial field and compare the results with those of nonnegative lasso \citep{wu_nl} and nonnegative elastic net \citep{wu_nen}. Through this example, (1) we provide general practice guidance for adapting ARGEN to solve real world problems; (2) We   demonstrate ARGEN's feasibility and flexibility compared to the existing methods. In particular, we highlight that ARGEN can deal with problems that require constraints on the coefficient, while none of the existing methods \citep{wu_nl,wu_nen} can.

In portfolio management, \enquote{how close is the constructed tracking portfolio's return compared to that of benchmark index} is a primary measurement for accessing portfolio performance for passive strategies such as index-tracking strategy. Hence, inspired by \cite{Santanna}, we evaluate the tracking portfolio performance from the following three perspectives. Our primary performance measurement is tracking error (TE),
$$
TE = \sqrt{\frac{\sum\limits_{t=1}^{T}\left( (r_{t}^{p} - r_{t}^{b}) - \mathbb E[r_{t}^{p} - r_{t}^{b}]\right)^2}{T}}
$$
to measure the volatility of the excess return of a portfolio to the corresponding benchmark. We also compute the annual volatility of portfolio return (ARV),
$$
ARV = \sqrt{252} \sqrt{\frac{\sum\limits_{t=1}^{T}\left(r_{t}^{p} - \mathbb E[r_{t}^{p}]\right)^2}{T}}
$$
to measure the annualized return volatility of a portfolio. In addition, we also report the cumulative return 
$$
CR = \prod_{t=1}^{T} (1+r_{t}^{p}) - 1
$$
of the construction portfolios in our study. Here $r_{t}^{p}$ denotes the portfolio return at time $t$, $r_{t}^{b}$ is the benchmark return at time $t$, and $T$ is the total number of periods. 

Because there is no guarantee that the normalized $\widehat\beta$ is still less than $t$, we introduce the following normalization process, which constrains the way of choosing the lower and upper limits of $\mathcal I$. Recall that $s_i$ and $t_i$ are the lower and upper bounds of the coefficient $\beta_i$. To guarantee that the portfolio weight (i.e., normalized $\beta_i$), denoted by $\tilde{\beta}_i$,  for stock $i$ satisfies $0 \le \tilde{\beta}_i \le t_i \le 1$, we need $s_i$ and $t_i$ satisfy the following:
\begin{equation*}
    t_i + \sum_{j\in\{1,\ldots,p\}\backslash\{i\}} s_j \ge 1,
\end{equation*}
because it yields
\begin{equation*}
\tilde{\beta}_i = \frac{\beta_i}{\sum_{j=1}^{p} \beta_j} \le \frac{t_i}{t_i+\sum_{j\in\{1,\ldots,p\}\backslash\{i\}} s_j}\le t_i.
\end{equation*}
In the special case of $s_i = s_0$ and $t_i = t_0$, we shall choose the lower and upper bounds through 
\begin{equation*}
    \frac{1-t_0}{p-1} \le s_0 \le t_0 \le 1.
\end{equation*}

We use 5-year (from February 19, 2016 to February 18, 2020) historical daily prices (1259 data points) of S\&P 500\footnote{Retrieved from finance.yahoo.com} as our benchmark index and those of the constituent equities\footnote{Retrieved from quandl.com}. Because the list of S\&P 500 constituents is updated regularly by S\&P Dow Jones Indices LLC, we only include the daily prices from 377 stocks that have not been changed during the period of interest. In the linear model (\ref{1}), $Y$ is the vector of the daily return of the S\&P 500 index and columns of $X$ are the daily returns of the 391 stocks. To follow a buy-and-hold investment strategy, we split the data into training, validation, and testing sets. The training and validation sets consist of the first $252$ data points (12 months), 20\% of which are in the validation set. The remaining $1006$ data points are referred to as the testing set. In addition, we construct a long-only portfolio by ensuring the lower bound $s$ assumes only nonnegative values.

In the following, we outline our procedure of the index tracking problem. First, we target selecting $N$ individual stocks to construct the tracking portfolio. The number $N$ is among the normal range of the number of stocks hold to remove risk exposure and avoid unnecessary transaction costs. In other words, we constrain the number of nonzero elements in $\widehat{\beta}$ to $N$. Thus we use the bisection search \citep{wu_nen} in Algorithm \ref{algo:bisection-search-lam_1}  to determine the optimal $\lambda_n^{(1)}$ that produce the right number of nonzero coefficients, given $N=30,~50,~70,~90$ respectively, $\lambda_n^{(2)}=0$, $\mathrm{w}_n$ has equal elements, and $\mathcal I=[0,+\infty)$, respectively. Hence we obtain the 50 stocks selected by the model. This first process proceeds on the training and validation sets.

\begin{algorithm}[h!]
\LinesNumbered 
\KwIn{ $\lambda_n^{(2)},\mathrm w_n, \Sigma_n, \mathcal I, N$.}
 \textbf{Initialization:} $\lambda_{\text{down}} \longleftarrow 0$; $\lambda_{\text{up}} \longleftarrow 1$\;
 $\lambda \longleftarrow (\lambda_{\text{up}} + \lambda_{\text{down}})/{2}$\;
    $\widehat{\beta} \longleftarrow\argmin_{\beta \in\mathcal I}\left( \left\| Y - X\beta \right\|_2^2+\lambda  \mathrm w_{n}'|\beta|+\lambda_n^{(2)}\beta'\mathrm\Sigma_n\beta\right)$\;
 
 \While{ $\|\widehat{\beta}\|_0 \neq N$ }{
 
  \eIf{ $\|\widehat{\beta}\|_0>N$ } {
       $\lambda_{\text{down}} \longleftarrow \lambda$\;
       $\lambda \longleftarrow (\lambda_{\text{up}} + \lambda_{\text{down}})/{2}$\;
    } {
       $\lambda_{\text{up}} \longleftarrow \lambda$\;
       $\lambda \longleftarrow (\lambda_{\text{up}} + \lambda_{\text{down}})/{2}$\;
  }
  $\widehat{\beta} \longleftarrow\argmin_{\beta \in\mathcal I}\left( \left\| Y - X\beta \right\|_2^2+\lambda  \mathrm w_{n}'|\beta|+\lambda_n^{(2)}\beta'\Sigma_n\beta\right)$\;
}
 \KwOut{$\lambda$ }
\caption{Bisection search for $\lambda_n^{(1)}$}
\label{algo:bisection-search-lam_1}
\end{algorithm}

Next, we consider $\mathcal I = [0.0082, 0.6]$ and $\mathcal I = [0.0041, 0.8]$, respectively, and  apply ARLS and ARGEN to the corresponding data set of the selected 30, 50, 70, and 90 stocks to experiment the ARGEN algorithm. The ARLS is viewed as the baseline. For ARGEN, we search $\lambda_n^{(1)}$ randomly in a range of $(10^{-8}, 5\times10^{-2})$, which is a smaller searching grid compared with that in Section \ref{Hyper_parameter_tuning}, since larger range results in over 50 vanishing coefficients. The $\lambda_n^{(2)}$ is randomly searched in a range of $(10^{-8}, 10^2)$. We take $w_{up}=1$, and $d_{up}=1$. The hyper-parameter tinning process is conducted in Optuna hyper-parameter optimization framework \cite{DBLP:journals/corr/abs-1907-10902}, and selected the parameter set that has the lowest validation score, measured by MSE, compared with that of ARLS, and then apply it on testing data set to evaluate and compare the out-of-sample performances between the portfolios obtained using ARGEN and ARLS. 

\subsection{Experimental Results}

We follow the procedure as elaborated in the previous session to construct multiple ARGEN and ARLS portfolios with different numbers of stocks, and different numbers of hyper-parameter tuning trails. The portfolios' testing performance is summarized in Table \ref{experiment_result_different_stock_num}.

Particularly, Table \ref{experiment_result_different_stock_num} illustrates the performance of different ARGEN and ARLS portfolios constructed with different coefficient boundary and stock numbers. Across different portfolio construction configurations, the ARGEN portfolios tend to have lower tracking errors and annualized return volatility than ARLS portfolios, while satisfying the coefficient boundary conditions. Even though ARGEN portfolios tends to have lower cumulative returns, but they are comparable with the S$\&$P 500 index cumulative return during the same period, except for 30-stock ARGEN portfolios. Portfolios with a wider range of constraints ($[0.0041, 0.8]$) tracks better than portfolios with narrower constraints range ($[0.0082, 0.6]$), which is expected behavior. Another expected behavior we can observe from the results is that as we increase the number of stocks in portfolios, the tracking errors decrease.

\begin{table}[H]
\caption{ARGEN vs ARLS tracking portfolio performance in the testing period. All the parameter sets are achieved the best score in the validation period using a relatively large number (30000) of hyper-parameter tuning tails. S$\&$P 500 index realizes a cumulative return (CR) as $66.4697\%$ and an annualized return volatility (ARV) as $20.6233\%$ in the testing period.}
\label{experiment_result_different_stock_num}
\centering
\begin{tabular}{ccccc|ccc}
\toprule 
 $\mathcal I$   & \# of &  & ARGEN &  &  & ARLS &  \tabularnewline
\cmidrule{3-8} 
  & Stocks & TE & ARV & CR & TE & ARV & CR \tabularnewline
\midrule
 $[0.0082, 0.6]$ & 30 & $2.20\%$ & $28.15\%$ & $56.70\%$ & $2.25\%$ & $29.14\%$ & $85.36\%$ \tabularnewline
 & 50 & $2.01\%$	& $24.34\%$ &	$97.70\%$ & $2.04\%$ & $24.93\%$ & $108.62\%$ \tabularnewline
  & 70 & $2.03\%$	& $24.80\%$ &	$94.86\%$ & $2.09\%$ & $25.96\%$ & $95.36\%$ \tabularnewline
  & 90 & $2.03\%$	& $24.69\%$ &	$83.34\%$ & $2.07\%$ & $25.59\%$ & $85.62\%$ \tabularnewline
\cmidrule{1-8} 
$[0.0041, 0.8]$& 30 & $2.24\%$	& $28.87\%$ &	$54.14\%$ & $2.23\%$ & $28.80\%$ & $80.58\%$ \tabularnewline
  & 50 & $1.82\%$	& $20.31\%$ &	$66.69\%$ & $2.01\%$ & $24.30\%$ & $112.86\%$ \tabularnewline
  & 70 & $1.93\%$	& $22.56\%$ &	$114.65\%$ & $1.97\%$ & $23.49\%$ & $110.33\%$ \tabularnewline
  & 90 & $1.93\%$	& $22.58\%$ &	$98.33\%$ & $1.95\%$ & $23.00\%$ & $95.51\%$ \tabularnewline
 \bottomrule 
\end{tabular}
\end{table}

\section{Conclusion and Future Perspectives}
\label{sec:conclusion}
In this paper, we propose the ARGEN for variable selection and regularization. ARGEN linearly combines generalized lasso and ridge penalties, which are $\mathrm{w}_n'\beta$ and $\beta'\Sigma_n\beta$, and it allows arbitrary lower and upper constraints on the coefficients. Many well-known methods including (nonnegative) lasso, ridge, and (nonnegative) elastic net are particular cases of ARGEN. We show that ARGEN has variable selection and estimation consistencies subject to some conditions.  We propose an algorithm to solve the ARGEN problem by applying multiplicative updates to a quadratic programming problem with a rectangle range and weighted $l_1$ regularizer (MU-QP-RR-W-$l_1$). The algorithm is implemented as a Python library through the PyPi server. The simulations and the application in index-tracking present shreds of evidence that ARGEN usually outperforms other methods discussed in the paper due to its flexibility and adaptability for problems with a small to moderate amount of predictors. In problems with a huge amount of predictors, although ARGEN should perform best theoretically, the cost might be high. In this situation ARLEN and ARREN might be better choices. We refer readers to the Github repository
\footnote{\url{https://github.com/songzhm/arbitraryElasticNet}} for full access to the code for the simulation and application parts.

Although in the paper the ARGEN penalty is added to linear models, there are possibilities that it is applied to other loss functions to improve their performances as directions of future research. Motivated by the index tracking problem, a constraint that guarantees the sum of weights equals one may be considered as another direction. Asymptotic behavior in law of the ARGEN estimator remains unknown. Most importantly, a more efficient tuning process is urgently required to apply ARGEN to solve more complicated problems. All of the above problems are open for future study.

\begin{appendix}
\section{Proof of Theorem \ref{thm:VSC}}
\label{proof_thm:VSC}
Without loss of generality, we assume that $\big(S_{(1-)}, S_{(1+)}, S_{(2)}, \ldots, S_{(6)}\big) =$ $(1, \ldots$ $, p)$ to simplify the notations. In addition, we follow the notations in Section \ref{sec:variable_selection}. In order to show that ARGEN admits variable selection consistency, we first need to show that the following result holds. 
\begin{lemma}
For $j\in\{1-,1+,2,\ldots,6\}$,
\label{prop}
\begin{equation}
\label{S_hat_low_bound}
    \mathbb P\Big(\widehat{S}_{(j)}\big(\lambda_n^{(1)},\lambda_n^{(2)},\mathrm w_n,\mathrm\Sigma_n\big) = S_{(j)}\Big|S_{(j)}\ne\emptyset\Big) \ge \mathbb P\left(\mathcal{E}(V_{(j)})\right),
\end{equation}
	for the events
		\begin{equation}
	\label{MU}
	\begin{cases}
	    \mathcal E(V_{(1-)}):=\left\{\rho( t_{(1-)}\wedge0)< V_{(1-)}< \rho( s_{(1-)}\wedge0)\right\},\\
	    \mathcal E(V_{(1+)}):=\left\{\rho( t_{(1+)}\vee0)< V_{(1+)}< \rho( s_{(1+)}\vee0)\right\},\\
	      \mathcal E(V_{(2)}):=\left\{V_{(2)}\le 0\right\},~\mathcal E(V_{(3)}):=\left\{V_{(3)}\ge 0\right\},\\
	      \mathcal E(V_{(4)}):=\left\{V_{(4)}\le 0\right\},~\mathcal E(V_{(5)}):=\left\{V_{(5)}\ge 0\right\},\\
	      \mathcal E(V_{(6)}):=\left\{-\mathrm w_{n,(6)}\le V_{(6)}\le \mathrm w_{n,(6)}\right\},
	\end{cases}
	\end{equation}
	where
	\begin{eqnarray}
&&V:=(V_1~\ldots~V_p)=(V_{(1)}~\ldots~V_{(6)}),\nonumber\\
	\label{Ui}
&& V_{(1)}\hspace{-0.1cm}:=\hspace{-0.1cm}\frac{1}{n} \Big(C_{11}+\frac{\lambda_n^{(2)}}{n}\Sigma_{n,(11)}\Big)^{-1}\hspace{-0.1cm}\Big(-X_{(1)}'\epsilon+\frac{\lambda_n^{(1)}}{2}
\begin{pmatrix}
-\mathrm w_{n,(1-)}\\
\mathrm w_{n,(1+)}
\end{pmatrix}+\lambda_n^{(2)}\left(\Sigma_{n,(12)} s_{(2)}+\Sigma_{n,(13)} t_{(3)}\right)\Big),\\
\label{V1-}
&&V_{(1-)}:=(V_i)_{i\in S_{(1-)}},~V_{(1+)}:=(V_i)_{i\in S_{(1+)}},\\
\label{Vj}
 &&V_{(j)}:=T_{(j)}-D_{(j)}\mathrm w_{n,(j)},~\mbox{for}~j=2,\ldots,6,
\end{eqnarray}
	with 
\begin{equation}
\begin{aligned}
\label{def:Tj}
&T_{(j)}:=\frac{2}{n}\left(X_{(j)}'X_{(1)}+\lambda_n^{(2)}\Sigma_{n,(j1)}\right)\Big(C_{11}+\frac{\lambda_n^{(2)}}{n}\Sigma_{n,(11)}\Big)^{-1}\\
&\times\Big(\frac{\lambda_n^{(2)}}{\lambda_n^{(1)}}\Sigma_{n,(11)}\beta_{(1)}^*-\frac{X_{(1)}'\epsilon}{\lambda_n^{(1)}}+\frac{1}{2}\diag(\sign(\beta_{(1)}^*))\mathrm w_{n,(1)}+\frac{\lambda_n^{(2)}}{\lambda_n^{(1)}}\left(\Sigma_{n,(12)} s_{(2)}+\Sigma_{n,(13)} t_{(3)}\right)\Big)\\
&-\frac{2\lambda_n^{(2)}}{\lambda_n^{(1)}}\Sigma_{n,(j1)}\beta_{(1)}^*+\frac{2}{\lambda_n^{(1)}}X_{(j)}'\epsilon-\frac{2\lambda_n^{(2)}}{\lambda_n^{(1)}}(\Sigma_{n,(j2)} s_{(2)}+\Sigma_{n,(j3)} t_{(3)}),
\end{aligned}
\end{equation}
\begin{eqnarray}
\label{D}
&&(D_{(2)}~\ldots~D_{(6)}):=(\diag(\sign( s_{(2)}))~\diag(\sign( t_{(3)}))~1~-1~0),\\
&& \rho = \Big(C_{11}+\frac{\lambda_n^{(2)}}{n}\Sigma_{n,(11)}\Big)^{-1}C_{11}\beta_{(1)}^*\in\mathbb R^{\# S_{(1)}},\nonumber\\
    &&\rho( u):=\left(\rho_{i}\right)_{i\in S_{(j)}}- u_{(j)},~~\mbox{for}~ u_{(j)}=(u_i)_{i\in S_{(j)}}~\mbox{and}~u\in\mathbb R^{p}\nonumber.
\end{eqnarray}
\end{lemma}
\begin{proof} 
Assume $\beta_{i}^*=0$ for some $i\in\{1,\ldots,p\}$. 
By the Karush-Kuhn-Tucker (KKT) conditions, for given $\lambda_n^{(1)}$, $\lambda_n^{(2)}$, $\mathrm w_n$, $\mathrm \Sigma_n$, (\ref{model}) is equivalent to solve $\widehat\beta$ from the following constraint optimization problem
\begin{equation}
\label{KKT}
	\min\limits_{\beta\in[ s,
 t],~\gamma\ge0,~\mu\ge0}\left(\|Y-X\beta\|^2_2 + \lambda_n^{(1)} \mathrm w_{n}' |\beta|+\lambda_n^{(2)}\beta'\Sigma_n \beta+\gamma'( s-\beta)+\mu'(\beta- t)\right).
\end{equation}
Since the term $|\beta|$ in (\ref{KKT}) is not differentiable but subdifferentiable at $0$, similar to the lasso problem (see Eqs. (2)-(9) in \cite{tibshirani2013lasso}), (\ref{KKT}) is equivalent to 
\begin{equation}
\label{KKT_1}
\begin{cases}
	-2X'(Y-X\widehat{\beta}) + \lambda_n^{(1)}\diag(\theta) \mathrm w_{n}+2\lambda_n^{(2)}\Sigma_n \widehat{\beta}-\gamma+\mu = 0, \\
	\theta_i\in\left\{
	\begin{array}{ll}
	     \{\sign(\widehat{\beta}_i)\}&~\mbox{if}~\widehat\beta_i\ne0,  \\
	      \{1\}&~\mbox{if}~\widehat\beta_i=s_i=0,  \\	       \{-1\}&~\mbox{if}~\widehat\beta_i=t_i=0,  \\
	     [-1,1]&~\mbox{if}~\widehat\beta_i=0\in(s_i,t_i),
	\end{array}\right.\mbox{for}~i=1,\ldots,p,\\
	\widehat\beta\in[ s, t],~\gamma\ge0,~\mu\ge0,~\gamma'( s-\widehat\beta)=0,~\mu'(\widehat\beta- t)=0,
	\end{cases}
\end{equation}
where $\theta\in\mathbb R^p$ is called a subgradient of the function $(x_1,\ldots,x_p)\mapsto|x_1|+\ldots+|x_p|$ at $x=\widehat\beta$.
Let $\widehat{\beta}=(\widehat\beta_{(1)}~\ldots~\widehat\beta_{(6)})'$ be the estimates of  $\beta^*=(\beta_{(1)}^*~\ldots~\beta_{(6)}^*)'$  respectively. Recall that $
Y=X\beta^*+\epsilon
$
with $\beta_{(2)}^*= s_{(2)}$, $\beta_{(3)}^*= t_{(3)}$, $\beta_{(4)}^*= s_{(4)}=0$, $\beta_{(5)}^*= t_{(5)}=0$, $\beta_{(6)}^*=0\in( s_{(6)}, t_{(6)})$. Plugging them into (\ref{KKT_1}) yields 
\begin{equation}
\label{KKT_2}
	\begin{cases}
	-2\begin{pmatrix}
X_{(1)}'\\
\vdots\\
X_{(6)}'
\end{pmatrix}
\left(\begin{pmatrix}
X_{(1)}'\\
\vdots\\
X_{(6)}'
\end{pmatrix}'
\begin{pmatrix}
\beta_{(1)}^*-\widehat{\beta}_{(1)}\\
 s_{(2)}-\widehat{\beta}_{(2)}\\
 t_{(3)}-\widehat{\beta}_{(3)}\\
-\widehat{\beta}_{(4)}\\
-\widehat{\beta}_{(5)}\\
-\widehat{\beta}_{(6)}
\end{pmatrix}
+\epsilon
\right)+ \lambda_n^{(1)} 
\diag(\theta)
\begin{pmatrix}
\mathrm w_{n,(1)}\\
\vdots\\
\mathrm w_{n,(6)}
\end{pmatrix}\\
\hspace{2cm}+2\lambda_n^{(2)}
\begin{pmatrix}
\Sigma_{n,(11)}&\ldots&\Sigma_{n,(16)}\\
\vdots&\ddots&\vdots\\
\Sigma_{n,(61)}&\ldots&\Sigma_{n,(66)}
\end{pmatrix}
\begin{pmatrix}
\widehat\beta_{(1)}\\
\vdots\\
\widehat\beta_{(6)}
\end{pmatrix}-
	\begin{pmatrix}
\gamma_{(1)}\\
\vdots\\
\gamma_{(6)}
\end{pmatrix}+
\begin{pmatrix}
\mu_{(1)}\\
\vdots\\
\mu_{(6)}
\end{pmatrix}=0, \\
\theta_i\in\left\{
	\begin{array}{ll}
	     \{\sign(\widehat{\beta}_i)\}&~\mbox{if}~\widehat\beta_i\ne0,  \\
	      \{1\}&~\mbox{if}~\widehat\beta_i=s_i=0,  \\
	       \{-1\}&~\mbox{if}~\widehat\beta_i=t_i=0,  \\
	     [-1,1]&~\mbox{if}~\widehat\beta_i=0\in(s_i,t_i),
	\end{array}\right.\mbox{for}~i=1,\ldots,p,\\
	\widehat\beta\in[ s, t],~\gamma\ge0,~\mu\ge0,~\gamma'( s-\widehat\beta)=0,~\mu'(\widehat\beta- t)=0,~ s_{(4)}=0,~ t_{(5)}=0,~0\in( s_{(6)}, t_{(6)}).
	\end{cases}
\end{equation}
If there exists $\widehat\beta$ that satisfies (\ref{KKT_2}) and 
\begin{eqnarray*}
&&\widehat\beta_{(1-)}\in\left(-\infty,0\right)\cap( s_{(1-)}, t_{(1-)}),\\
&&\widehat\beta_{(1+)}\in\left(0,+\infty\right)\cap( s_{(1+)}, t_{(1+)}),\\
&&\widehat\beta_{(2)}= s_{(2)}\ne0,~\widehat\beta_{(3)}= t_{(3)}\ne0,~\widehat\beta_{(4)}=\widehat\beta_{(5)}=\widehat\beta_{(6)}=0,
\end{eqnarray*}
then $S_{(j)}=\widehat S_{(j)}(\lambda_n^{(1)},\lambda_n^{(2)},\mathrm w_n,\Sigma_n)$ for $j\in\{1-,1+,2,\ldots,6\}$. That makes (\ref{KKT_2}) equivalent to
\begin{equation}
\label{KKT_3}
	\begin{cases}
		-2
X_{(1)}'
\left(
X_{(1)}(
\beta_{(1)}^*-\widehat{\beta}_{(1)})
+\epsilon
\right)+ \lambda_n^{(1)} 
\begin{pmatrix}
-\mathrm w_{n,(1-)}\\
\mathrm w_{n,(1+)}
\end{pmatrix}+2\lambda_n^{(2)}\left(\Sigma_{n,(11)} \widehat\beta_{(1)}+\Sigma_{n,(12)}  s_{(2)}+\Sigma_{n,(13)}  t_{(3)}\right)=0,\\
	-2X_{(2)}'\left(
X_{(1)}(
\beta_{(1)}^*-\widehat{\beta}_{(1)})
+\epsilon
\right)+ \lambda_n^{(1)}\diag(\sign( s_{(2)}))\mathrm w_{n,(2)}\\
\hspace{3cm}+2\lambda_n^{(2)}\left(\Sigma_{n,(21)} \widehat\beta_{(1)}+\Sigma_{n,(22)}  s_{(2)}+\Sigma_{n,(23)}  t_{(3)}\right)=\gamma_{(2)},\\
	-2X_{(3)}'\left(
X_{(1)}(
\beta_{(1)}^*-\widehat{\beta}_{(1)})
+\epsilon
\right) + \lambda_n^{(1)}\diag(\sign( t_{(3)}))\mathrm w_{n,(3)}\\
\hspace{3cm}+2\lambda_n^{(2)}\left(\Sigma_{n,(31)} \widehat\beta_{(1)}+\Sigma_{n,(32)}  s_{(2)}+\Sigma_{n,(33)}  t_{(3)}\right)=-\mu_{(3)},\\
		-2X_{(4)}'\left(
X_{(1)}(
\beta_{(1)}^*-\widehat{\beta}_{(1)})
+\epsilon
\right) + \lambda_n^{(1)}\mathrm w_{n,(4)}+2\lambda_n^{(2)}\left(\Sigma_{n,(41)} \widehat\beta_{(1)}+\Sigma_{n,(42)}  s_{(2)}+\Sigma_{n,(43)}  t_{(3)}\right)=\gamma_{(4)},\\
	-2X_{(5)}'\left(
X_{(1)}(
\beta_{(1)}^*-\widehat{\beta}_{(1)})
+\epsilon
\right) - \lambda_n^{(1)}\mathrm w_{n,(5)}+2\lambda_n^{(2)}\left(\Sigma_{n,(51)} \widehat\beta_{(1)}+\Sigma_{n,(52)}  s_{(2)}+\Sigma_{n,(53)}  t_{(3)}\right)=-\mu_{(5)},\\
-2X_{(6)}'\left(
X_{(1)}(
\beta_{(1)}^*-\widehat{\beta}_{(1)})
+\epsilon
\right) + \lambda_n^{(1)}\diag(\theta_{(6)})\mathrm w_{n,(6)}+2\lambda_n^{(2)}\left(\Sigma_{n,(61)} \widehat\beta_{(1)}+\Sigma_{n,(62)}  s_{(2)}+\Sigma_{n,(63)}  t_{(3)}\right)=0,\\
	\theta_i\in\left\{
	\begin{array}{ll}
	     \{\sign(\widehat{\beta}_i)\}&~\mbox{if}~\widehat\beta_i\ne0,  \\
	      \{1\}&~\mbox{if}~\widehat\beta_i=s_i=0,  \\
	       \{-1\}&~\mbox{if}~\widehat\beta_i=t_i=0,  \\
	     [-1,1]&~\mbox{if}~\widehat\beta_i=0\in(s_i,t_i),
	\end{array}\right.\mbox{for}~i=1,\ldots,p,\\
	\widehat\beta\in[ s, t],~\gamma\ge0,~\mu\ge0.
	\end{cases}
\end{equation}
Solving $\widehat\beta_{(1)}$ from  the first equation in  (\ref{KKT_3}), we obtain
\begin{eqnarray}
\label{hat_beta_(1)}
\widehat\beta_{(1)}=\Big(C_{11}+\frac{\lambda_n^{(2)}\Sigma_{n,(11)}}{n}\Big)^{-1}\hspace{-0.1cm}\bigg(C_{11}\beta_{(1)}^*+\frac{X_{(1)}'\epsilon}{n}-\frac{\lambda_n^{(1)}\hspace{-0.1cm}
\begin{pmatrix}
-\mathrm w_{n,(1-)}\\
\mathrm w_{n,(1+)}
\end{pmatrix}}{2n}-\frac{\lambda_n^{(2)}(\Sigma_{n,(12)} s_{(2)}+\Sigma_{n,(13)} t_{(3)})}{n}\bigg),
\end{eqnarray}
where $C_{11}$ is defined in (\ref{def:C}). 
Replacing $\widehat\beta_{(1)}$ with (\ref{hat_beta_(1)}) in the rest equations in (\ref{KKT_3}) yields
\begin{equation}
    \label{V:rest}
    \begin{cases}
        T_{(2)}=\diag(\sign( s_{(2)}))\mathrm w_{n,(2)}-\frac{\gamma_{(2)}}{\lambda_{n}^{(1)}},\\
        T_{(3)}=\diag(\sign( t_{(3)}))\mathrm w_{n,(3)}+\frac{\mu_{(3)}}{\lambda_{n}^{(1)}},\\
        T_{(4)}=\mathrm w_{n,(4)}-\frac{\gamma_{(4)}}{\lambda_{n}^{(1)}},\\
         T_{(5)}=-\mathrm w_{n,(5)}+\frac{\mu_{(5)}}{\lambda_{n}^{(1)}},\\
          T_{(6)}=\diag(\theta_{(6)})\mathrm w_{n,(6)},
    \end{cases}
\end{equation}
where $T_{(2)},\ldots,T_{(6)}$ are defined in (\ref{def:Tj}). It follows from (\ref{V:rest}) that (\ref{KKT_2}) admits a solution $\widehat\beta$ which satisfies the variable selection consistency if and only if
\begin{equation}
    \label{condition_beta_1}
    \begin{cases}
        V_{(1-)}\in \left( \rho(0\wedge t_{(1-)}), \rho(0\wedge s_{(1-)})\right),\\
        V_{(1+)}\in \left( \rho(0\vee t_{(1+)}), \rho(0\vee s_{(1+)})\right),\\
        V_{(2)}\le0,~V_{(3)}\ge0,~V_{(4)}\le0,~V_{(5)}\ge0,\\
        V_{(6)}=\diag(\theta_{(6)})\mathrm w_{n,(6)}\in[-\mathrm w_{n,(6)},\mathrm w_{n,(6)}],
    \end{cases}
\end{equation}
where $V_{(j)}$'s are given in (\ref{V1-}) and (\ref{Vj}). Observing that (\ref{condition_beta_1}) is nothing else but the events $\mathcal E(V_{(j)})$'s defined in (\ref{MU}), therefore, (\ref{S_hat_low_bound}) holds if and only if (\ref{MU}) occurs. Hence Lemma \ref{prop} is proved.

\end{proof}

The lemma above provides a lower bound on the probability of ARGEN correctly selecting variables. To prove the theorem, we also need the following well-known result on the upper bound of the maximum of Gaussian random vector (see (3.6) in \cite{ledoux2013probability}).
\begin{lemma}
\label{lem:Ledoux}
Let $(X_1,\ldots,X_n)$ be any Gaussian random vector. For $n$ large enough, we have 
\begin{equation*}
    \mathbb E\big[\max_{1\le i\le n}|X_i|\big]\le 8\sqrt{\log n}\max_{1\le i\le n}\sqrt{\mathbb E [X_i]^2}.
\end{equation*}
\end{lemma}

With the results of Lemma \ref{prop} and \ref{lem:Ledoux}, we provide the proof of Theorem \ref{thm:VSC} in the following. \\
By Lemma \ref{prop}, to prove the theorem it suffices to build
\begin{equation}
\label{conv_MU}
    \mathbb P\left(\mathcal E(V_{(j)})\right)\xrightarrow[n\to\infty]{}1,~\mbox{for}~j=1,\ldots,6.
\end{equation}
Below, we show (\ref{conv_MU}) holds for each $j$. In the case of $j=1$,
 we need to show  both
$
    \mathbb P\big(\mathcal E(V_{(1-)})\big)\xrightarrow[n\to\infty]{}1
$
and
$
    \mathbb P\big(\mathcal E(V_{(1+)})\big)\xrightarrow[n\to\infty]{}1
$
hold. To obtain the former, we denote the $i$th element of $V_{(1)}$ in (\ref{Ui}) by $V_i$ and split it into
\begin{equation}
\label{U_split}
V_i=V_i^{(1)}+V_{i}^{(2)}.
\end{equation}
Here
\begin{equation}
    \label{U1}
    V_i^{(1)}:=\frac{e_i'}{n} \Big(C_{11}+\frac{\lambda_n^{(2)}}{n}\Sigma_{n,(11)}\Big)^{-1}\left(-X_{(1)}'\epsilon\right)
\end{equation}
is a Gaussian random variable with
$\mathbb E(V_i^{(1)})=0$ and
\begin{eqnarray*}
\mathbb Var\big(V_i^{(1)}\big)&=&\frac{\sigma^2}{n^2}e_i' \Big(C_{11}+\frac{\lambda_n^{(2)}}{n}\Sigma_{n,(11)}\Big)^{-2}C_{11}e_i\le\frac{\sigma^2\#S_{(1)}\operatorname{trace}(C_{11})}{n^2\big(\Lambda_{\min} (C_{11}+\lambda_n^{(2)}\Sigma_{n,(11)}/n)\big)^2},
\end{eqnarray*}
where
$\Lambda_{\min}(\bullet)$ denotes the minimal eigenvalue and $\operatorname{trace}(\bullet)$ denotes the trace of the matrix. 
The second term $V_i^{(2)}$ can be bounded by:
 \begin{eqnarray}
    \label{U2}
    V_i^{(2)}\hspace{-0.2cm}:=\frac{e_i'}{n}\Big(C_{11}+\frac{\lambda_n^{(2)}\Sigma_{n,(11)}}{n}\Big)^{-1}\hspace{-0.1cm}\Big(\frac{\lambda_n^{(1)}}{2}
\begin{pmatrix}
-\mathrm w_{n,(1-)}\\
\mathrm w_{n,(1+)}
\end{pmatrix}+\lambda_n^{(2)}(\Sigma_{n,(12)} s_{(2)}+\Sigma_{n,(13)} t_{(3)})\Big)\hspace{-0.1cm}\in\hspace{-0.1cm}\left[\frac{C_n^{\min}}{n},\hspace{-0.1cm}\frac{C_n^{\max}}{n}\right].
\end{eqnarray}
where $C_n^{\min}$ and $C_n^{\max}$ are defined in (\ref{rho_C}), and $e_i$ is a vector with the $i$th element be $1$ and the others be $0$. 
Elementary probability calculus shows
\begin{equation}
\begin{aligned}
\label{MU_split}
&\mathbb P\left(\mathcal E(V_{(1-)})\right)\ge \mathbb P\left(\left\{\min V_{(1-)}>\max\rho( t_{(1-)}\wedge0)\right\} \cap \left\{\max V_{(1-)}< \min\rho( s_{(1-)}\wedge0)\right\}\right)\\
&= 1-\mathbb P\left(\min V_{(1-)}\le\max\rho( t_{(1-)}\wedge0)\right)-\mathbb P\left(\max V_{(1-)}\ge \min\rho( s_{(1-)}\wedge0)\right)\\
&\hspace{0.5cm}+\mathbb P\left(\left\{\min V_{(1-)}\le\max\rho( t_{(1-)}\wedge0)\right\}\cap\left\{\max V_{(1-)}\ge \min\rho( s_{(1-)}\wedge0)\right\}\right).
\end{aligned}
\end{equation}
We observe that for large $n$, $\max\rho( t_{(1-)}\wedge0)<0$. It results from (\ref{con_2}), (\ref{U_split}), (\ref{U1}), (\ref{U2}), and Lemma \ref{lem:Ledoux} that
    \begin{equation}
\begin{aligned} \label{upperbound_U1}
      &\mathbb P\left(\min V_{(1-)}\le \max\rho( t_{(1-)}\wedge0)\right)=\mathbb P\Big(\max\big\{- V_{(1-)}\big\}\ge- \max\rho( t_{(1-)}\wedge0)\Big)\\
      &\le  \mathbb P\Big(\max_{i\in S_{(1-)}}\big\{-V_i^{(1)}\big\}-\frac{C_n^{\min}}{n}\ge- \max\rho( t_{(1-)}\wedge0)\Big)  \\
       &=  \mathbb P\Big(\frac{1}{|\max\rho( t_{(1-)}\wedge0)|}\big(\max_{i\in S_{(1-)}}V_i^{(1)}-\frac{C_n^{\min}}{n}\big)\ge 1\Big) \\
       &\le \frac{1}{|\max\rho( t_{(1-)}\wedge0)|}\Bigg(8\sqrt{\frac{\sigma^2\# S_{(1)}\operatorname{trace}(C_{11})\log (\#S_{(1-)})}{n^2\left(\Lambda_{\min} (C_{11}+\lambda_n^{(2)}\Sigma_{n,(11)}/n)\right)^2}}+\frac{|C_n^{\min}|}{n}\Bigg)\xrightarrow[n\to\infty]{}0.
    \end{aligned}
\end{equation}
    Similarly, because $\min\rho( s_{(1-)}\wedge0)>0$ holds for large $n$, it follows from (\ref{con_3}), (\ref{U_split}), (\ref{U1}), (\ref{U2}), and Lemma \ref{lem:Ledoux} that
    \begin{equation}
\begin{aligned}
     \label{lowerbound_U1}
     &\mathbb P\left(\max V_{(1-)}\ge \min\rho( s_{(1-)}\wedge0)\right)\\
     &\le  \mathbb P\Big(\max_{i\in S_{(1-)}}\big\{V_i^{(1)}\big\}+\frac{C_n^{\max}}{n}\ge\min\rho( s_{(1-)}\wedge0)\Big)  \\
       &=  \mathbb P\Big(\frac{1}{\min\rho( s_{(1-)}\wedge0)}\big(\max_{i\in S_{(1-)}}V_i^{(1)}+\frac{C_n^{\max}}{n}\big)\ge 1\Big) \\
       &\le \frac{1}{\min\rho( s_{(1-)}\wedge0)}\Bigg(8\sqrt{\frac{\sigma^2\# S_{(1)}\operatorname{trace}(C_{11})\log (\# S_{(1-)})}{n^2\left(\Lambda_{\min} (C_{11}+\lambda_n^{(2)}\Sigma_{n,(11)}/n)\right)^2}}+\frac{|C_n^{\max}|}{n}\Bigg)\xrightarrow[n\to\infty]{}0.
  \end{aligned}
\end{equation}
 Hence, (\ref{conv_MU}) with $j=1-$  results from (\ref{MU_split}), (\ref{upperbound_U1}) and (\ref{lowerbound_U1}). The case of $j=1+$ can be proved following quite a similar way, so we omit the details. 
 
To show $\mathbb P\big(\mathcal E(V_{(2)})\big)\xrightarrow[n\to\infty]{}1$,
 we first define
 \begin{eqnarray*}
    %  \label{recall_AREIC}
     &&V_{(2)}^{(1)}:=\Big(C_{21}+\frac{\lambda_n^{(2)}}{n}\Sigma_{n,(21)}\Big)\Big(C_{11}+\frac{\lambda_n^{(2)}}{n}\Sigma_{n,(11)}\Big)^{-1}\nonumber\\
&&\times\Big(\frac{2\lambda_n^{(2)}}{\lambda_n^{(1)}}\Sigma_{n,(11)}\beta_{(1)}^*+
\begin{pmatrix}
-\mathrm w_{n,(1-)}\\
\mathrm w_{n,(1+)}
\end{pmatrix}+\frac{2\lambda_n^{(2)}}{\lambda_n^{(1)}}\left(\Sigma_{n,(12)} s_{(2)}+\Sigma_{n,(13)} t_{(3)}\right)\Big)\nonumber\\
&&-\frac{2\lambda_n^{(2)}}{\lambda_n^{(1)}}\Sigma_{n,(21)}\beta_{(1)}^*-2\frac{\lambda_n^{(2)}}{\lambda_n^{(1)}}(\Sigma_{n,(22)} s_{(2)}+\Sigma_{n,(23)} t_{(3)})\le \diag(\sign( s_{(2)}))\mathrm w_{n,(2)}-\eta_{(2)}.\nonumber\\
 \end{eqnarray*}
Setting $\eta_0=\min\{\eta_{(2)}\}$, by the AREIC (\ref{AREIC_ineq}) we know
 $
    %  \label{def_V1}
     V_{(2)}^{(1)}\le \diag(\sign( s_{(2)}))$ $\mathrm w_{n,(2)} - \eta_0.
 $
 Accordingly, let
 \begin{eqnarray*}
    %  \label{def_V2} 
     &&V_{(2)}^{(2)}:=V_{(2)}+\diag(\sign( s_{(2)}))\mathrm w_{n,(2)}-V_{(2)}^{(1)}\nonumber\\
     &&=\Bigg\{\Big(C_{21}+\frac{\lambda_n^{(2)}}{n}\Sigma_{n,(21)}\Big)\Big(C_{11}+\frac{\lambda_n^{(2)}}{n}\Sigma_{n,(11)}\Big)^{-1}\Big(-\frac{2X_{(1)}'}{\lambda_n^{(1)}}\Big)+\frac{2X_{(2)}'}{\lambda_n^{(1)}}\Bigg\}\epsilon.
 \end{eqnarray*}
 Observe that for each coordinate index $j$, $e_j'V_{(2)}^{(2)}$ is a zero mean Gaussian random variable and with $n$ large enough, the variance is bounded above by
 $
 \mathbb E\big[e_j'V_{(2)}^{(2)}\big]^2\le 4n\sigma^2/(\lambda_n^{(1)})^2.
 $
 Then it follows that
 \begin{equation*}
     1-\mathbb P\left(\mathcal E(V_{(2)})\right)=\mathbb P\left(V_{(2)}>0\right)\le \mathbb P\big(\max V_{(2)}^{(2)}>\eta_0\big). 
 \end{equation*}
 By Markov's inequality and Lemma \ref{lem:Ledoux}, we obtain
 \begin{eqnarray*}
 &&\mathbb P\Big(\max V_{(2)}^{(2)}>\eta_0\Big)\le \mathbb P\Big(\max \big|V_{(2)}^{(2)}\big|>\eta_0\Big)\le\frac{\mathbb E\Big[\max\big|V_{(2)}^{(2)}\big|\Big]}{\eta_0}\nonumber\\
 &&\le \frac{8\sqrt{\log(\#S_{(2)})}}{\eta_0}\max_{j\in S_{(2)}}\sqrt{\mathbb E\Big[e_j'V_{(2)}^{(2)}\Big]^2}\le \frac{16\sqrt{\log(\#S_{(2)})}}{\eta_0}\frac{\sqrt{n}\sigma}{\lambda_n^{(1)}}.
 \end{eqnarray*}
Hence, $\mathbb P\big(\mathcal E(V_{(2)})\big)\xrightarrow[n\to\infty]{}1$ follows from (\ref{con_1}). The same process with slight modifications will be followed for the other cases of $j=3,4,5,6$. Therefore, (\ref{conv_MU}) holds and we have completed the proof.

\section{Proof of Theorem \ref{thm:conistency_algo}}
\label{sec:proof_thm_conistency_algo}
First, (\ref{U_ineq}) holds if for any $u,v\in[0, l]$, the auxiliary function $G(\bullet,\bullet)$ satisfies the following:
\begin{equation}
\label{auxiliary_fcn_property_1}
    F(v)= G(v, v),
\end{equation}
\begin{equation}
\label{auxiliary_fcn_property_2}
    F(u)\leq G(u, v).
\end{equation}
 This is because 
$
F(U(v))\leq G(U(v),v)\leq G(v,v)= F(v).
$ In view of (\ref{4}) and (\ref{def::auxiliary function}), Equation (\ref{auxiliary_fcn_property_1}) can be easily obtained through plugging $u=v$ into $ G(u,v)$. To prove  (\ref{auxiliary_fcn_property_2}), we need the following preliminary results of Lemmas 1 and 2 in \cite{sha}:
\begin{eqnarray}
\label{ref::lem1}
  && \frac{1}{2} u'A^+u\leq  \frac{1}{2} \sum_{1\le i,j\le p}\frac{A^+_{ij}u_i^2v_j}{v_i},\\
\label{ref::lem2}
    &&-\frac{1}{2} u'A^-u\leq -\frac12\sum_{1\le i,j\le p}A^-_{ij}v_iv_j\Big(1+\log\frac{u_iu_j}{v_iv_j}\Big).
\end{eqnarray}
Then by (\ref{auxiliary_fcn_property_1}), we have $$
F(u)=G(u,u)=\frac{1}{2} u'A^+u -\frac{1}{2} u'A^-u+\sum_{i=1}^p (b_iu_i+d_i|u_i-v^0_i|),
$$
which is bounded above by $G(u,v)$ using (\ref{ref::lem1}) and (\ref{ref::lem2}).

Next, to show that (\ref{U_eq}) and (ii) hold, it suffices to prove that $U(v)$ is the unique vector in $[0, l]$ that minimizes $G(\bullet,v)$, and the mapping $U(\bullet)$ has the form as (\ref{eq:noboundIter}). Observe that $G(u,v)$ can be rewritten as
\begin{equation*}
% \label{G:equiv}
G(u,v)=\sum_{i=1}^p G_i(u_i)-\frac12\sum_{1\le i,j\le p}A^-_{ij}v_iv_j,
\end{equation*}
where 
$$
G_i(u_i):=\frac{1}{2}\Big(\sum_{j=1}^pA^+_{ij}v_j\Big)\frac{u_i^2}{v_i}-\Big(\sum_{j=1}^pA^-_{ij}v_j\Big)v_i\log\frac{u_i}{v_i}+b_iu_i+d_i|u_i-v^0_i|.
$$
Because the term $(1/2)\sum_{1\le i,j\le p}A^-_{ij}v_iv_j$ is independent of $u$, the minimization of $G(u,v)$ over $u$ can be accomplished by minimizing $G_i(u_i)$ over the marginal variable $u_i$, for each $i=1,\ldots,p$.
Fixing $i\in\{1,\ldots,p\}$, below we minimize $G_i(u_i)$ over $u_i\in[0,l_i]$ in two cases.

First we consider the case when $u_i\in[0,l_i]\cap[v_i^0,+\infty)$. In this case  $G_i(u_i)$ becomes
   \begin{equation*}
%   \label{G_i_case_1}
G_i(u_i)=\frac{1}{2}\Big(\sum_{j=1}^pA^+_{ij}v_j\Big)\frac{u_i^2}{v_i}-\Big(\sum_{j=1}^pA^-_{ij}v_j\Big)v_i\log\frac{u_i}{v_i}+b_iu_i+d_i(u_i-v^0_i)
\end{equation*} 
and it is differentiable over $\mathbb R$. Taking the first and second derivatives of $G_i(\bullet)$, we obtain
\begin{equation*}
% \label{G_1_case_1_diff}
G_i'(u_i) = \Big(\sum_{j=1}^pA^+_{ij}v_j\Big)\frac{u_i}{v_i}-\Big(\sum_{j=1}^pA^-_{ij}v_j\Big)\frac{v_i}{u_i}+b_i+d_i
\end{equation*}
and
\begin{equation*}
% \label{G_1_case_1_diff_2}
G_i''(u_i) = \Big(\sum_{j=1}^pA^+_{ij}v_j\Big)\frac{1}{v_i}+\Big(\sum_{j=1}^pA^-_{ij}v_j\Big)\frac{v_i}{u_i^2}.
\end{equation*}
Here  $u,v\ge0$ and $A$ is strictly positive definite, so $\sum_{j=1}^pA^+_{ij}v_j$ and $\sum_{j=1}^pA^-_{ij}v_j$ cannot be simultaneously equal to 0. Therefore $G_i''(u_i)>0$ for all $u_i\in[0,+\infty)$ and hence $G_i(\bullet)$ is strictly convex over $[0,+\infty)$. Then the minimum of $G_i(\bullet)$ over $[0,+\infty)$ is obtained on the unique critic point $r_1$ such that $G_i'(r_1)=0$, which yields
\begin{equation*}
% \label{r_1}
r_1= \frac{-(b_i+d_i)+\sqrt{(b_i+d_i)^2+4a_i(v)c_i(v)}}{2a_i(v)}v_i.
\end{equation*}
It follows that the minimum of $G_i(\bullet)$ over $[0,l_i]\cap [v_i^0,+\infty)$ is given below:
\begin{equation}
    \label{case_1_result}
    \argmin_{u_i\in[0,l_i]\cap[v_i^0,+\infty)}G_i(u_i)=\left\{
    \begin{array}{ll}
    v_i^0&~\mbox{if}~r_1\le v_i^0\le l_i,\\
   r_1&~\mbox{if}~v_i^0<r_1 \le l_i,\\
   l_i&~\mbox{if}~v_i^0\le l_i<r_1,\\
   \emptyset&~\mbox{if}~v_i^0> l_i.
    \end{array}\right.
\end{equation}
Next, we discuss the other case. When $u_i\in[0,l_i]\cap[0,v_i^0)$, we have
   \begin{equation*}
%   \label{G_i_case_2}
G_i(u_i)=\frac{1}{2}\Big(\sum_{j=1}^pA^+_{ij}v_j\Big)\frac{u_i^2}{v_i}-\Big(\sum_{j=1}^pA^-_{ij}v_j\Big)v_i\log\frac{u_i}{v_i}+b_iu_i-d_i(u_i-v^0_i).
\end{equation*} 
Taking the first and second derivatives of it yields
\begin{equation*}
G_i'(u_i) = \Big(\sum_{j=1}^pA^+_{ij}v_j\Big)\frac{u_i}{v_i}-\Big(\sum_{j=1}^pA^-_{ij}v_j\Big)\frac{v_i}{u_i}+b_i-d_i
\end{equation*}
and
\begin{equation*}
G_i''(u_i) = \Big(\sum_{j=1}^pA^+_{ij}v_j\Big)\frac{1}{v_i}+\Big(\sum_{j=1}^pA^-_{ij}v_j\Big)\frac{v_i}{u_i^2}>0.
\end{equation*}
Because $G_i(\bullet)$ is strictly convex over $[0,+\infty)$, the minimum of $G_i(\bullet)$ over $[0,+\infty)$ is uniquely obtained at $r_2$ such that $G_i'(r_2)=0$, that is,
\begin{equation*}
r_2= \frac{-(b_i-d_i)+\sqrt{(b_i-d_i)^2+4a_i(v)c_i(v)}}{2a_i(v)}v_i.
\end{equation*}
It follows that the minimum of $G_i(\bullet)$ over $[0,l_i]\cap[0,v_i^0)$ is given as:
\begin{equation}
    \label{case_2_result}
    \argmin_{u_i\in[0,l_i]\cap[0,v_i^0)}G_i(u_i)=\min\left\{r_2,l_i,v_i^0\right\}.
\end{equation}

Combining the two cases (\ref{case_1_result}) and (\ref{case_2_result}), we obtain
\begin{equation*}
 \argmin_{u_i\in[0,l_i]}G_i(u_i)=\left\{
\begin{array}{cc}
   \min\{r_1,l_i\} & \mbox{if~} r_1>v_i^0, \\
   \min\{r_2,l_i\} & \mbox{if~} r_2<v_i^0, \\
    \min\{v_i^0,l_i\} & \mbox{otherwise.}
\end{array}\right.    
\end{equation*}
Denote by $\widetilde U:~\mathbb R^p\mapsto \mathbb R^p$ such that  for $i=1,\ldots,p$,
\begin{equation}
\label{def:U_tilde}
\big(\widetilde U(v)\big)_i:=\left\{
\begin{array}{cc}
   \min\{r_1,l_i\} & \mbox{if~} r_1>v_i^0, \\
   \min\{r_2,l_i\} & \mbox{if~} r_2<v_i^0, \\
    \min\{v_i^0,l_i\} & \mbox{otherwise}.
\end{array}\right.
\end{equation}
We conclude that for each $v\in\mathbb R_+^p$, the vector given in (\ref{def:U_tilde}) is unique on $[0, l]$ satisfying
    $
    G\big(\widetilde U(v),v\big)=\min_{u\in[0, l]}G(u,v).
    $
    This proves (\ref{U_eq}). In addition, by the uniqueness of the minimizer of $G(\bullet,v)$ over $[0,l]$, the mapping $U(\bullet)$ should have the same form as $\widetilde U(\bullet)$ given in (\ref{def:U_tilde}), which is exactly  (\ref{eq:noboundIter}). This proves (ii) in Theorem \ref{thm:conistency_algo}.
\end{appendix}

%%%%%%%%%%%%%%%%%%%%%%%%%%%%%%%%%%%%%%%%%%%%%%
%% Support information, if any,             %%
%% should be provided in the                %%
%% Acknowledgements section.                %%
%%%%%%%%%%%%%%%%%%%%%%%%%%%%%%%%%%%%%%%%%%%%%%
%\begin{acks}[Acknowledgments]
% The authors would like to thank ...
%\end{acks}
%%%%%%%%%%%%%%%%%%%%%%%%%%%%%%%%%%%%%%%%%%%%%%
%% Funding information, if any,             %%
%% should be provided in the                %%
%% funding section.                         %%
%%%%%%%%%%%%%%%%%%%%%%%%%%%%%%%%%%%%%%%%%%%%%%
%\begin{funding}
% The first author was supported by ...
%
% The second author was supported in part by ...
%\end{funding}


%% if your bibliography is in bibtex format, uncomment commands:
\bibliographystyle{plain} % Style BST file (imsart-number.bst or imsart-nameyear.bst)
\bibliography{ref}       % Bibliography file (usually '*.bib')

%% or include bibliography directly:
% \begin{thebibliography}{}
% \bibitem{b1}
% \end{thebibliography}    
    
\end{document}